\documentclass{article} %
\usepackage[preprint]{colm2026_conference}

\usepackage{microtype}
\usepackage{hyperref}
\usepackage{url}
\usepackage{bbm}
\usepackage{amsmath}
\usepackage{amsthm}
\usepackage{amssymb}
\usepackage{pifont}
\usepackage{tikz}
\usetikzlibrary{backgrounds,fit,arrows.meta}
\usepackage{subcaption}
\usepackage{arydshln}
\usepackage{xcolor}
\usepackage{colortbl}
\usepackage{booktabs}
\usepackage{multirow}
\usepackage{todonotes}
\usepackage{nicematrix}
\usepackage{wrapfig}
\usepackage[most]{tcolorbox}
\usepackage{xcolor}
\usepackage{enumitem}

\definecolor{promptred}{RGB}{170,0,0}
\definecolor{promptbg}{RGB}{248,242,232}
\definecolor{promptblue}{RGB}{0,0,170}
\definecolor{promptgreen}{RGB}{0,170,0}
\definecolor{evalheader}{RGB}{220, 235, 252}  %
\definecolor{evalEasy}{RGB}{212, 237, 218}     %
\definecolor{evalMedium}{RGB}{255, 243, 205}   %
\definecolor{evalHard}{RGB}{255, 218, 185}     %
\definecolor{evalExtreme}{RGB}{248, 200, 200}  %
\definecolor{evalAll}{RGB}{225, 225, 235}       %

\newtheorem{lemma}{Lemma}

\usepackage{lineno}

\definecolor{darkblue}{rgb}{0, 0, 0.5}
\hypersetup{colorlinks=true, citecolor=darkblue, linkcolor=darkblue, urlcolor=darkblue}

\usepackage{color-edits}
\addauthor{ag}{orange}
\addauthor{gn}{blue}
\addauthor{ak}{green}

\newcommand{\method}{RAO}
\newcommand{\methodlong}{Recursive Agent Optimization}
\newcommand{\textcraft}{\textsc{Textcraft-Synth}}
\newcommand{\oolong}{\textsc{Oolong-Real}}
\newcommand{\deepdive}{\textsc{DeepDive}}
\newcommand{\arte}{\textsc{ART-E}}

\newcommand{\cmark}{\textcolor{green!60!black}{\ding{51}}}
\newcommand{\xmark}{\textcolor{red!70!black}{\ding{55}}}
\newcommand{\qmark}{\textcolor{gray}{\textbf{?}}}

\title{Recursive Agent Optimization}

\author{Apurva Gandhi$^1$ \quad Satyaki Chakraborty$^2$ \quad Xiangjun Wang$^2$ \\ \textbf{Aviral Kumar$^1$} \quad \textbf{Graham Neubig$^1$} \\
$^1$Carnegie Mellon University \& $^2$Amazon AGI Labs \\
{\tt \{apurvag,aviralku,gneubig\}@cs.cmu.edu} \\ \tt \{satyaki,xjai\}@amazon.com}

\begin{document}

\ifcolmsubmission
\linenumbers
\fi

\maketitle

\begin{abstract}
We introduce \textit{\methodlong~(\method)}, a reinforcement learning approach for training \emph{recursive agents}: agents that can spawn and delegate sub-tasks to new instantiations of themselves recursively. Recursive agents implement an inference-time scaling algorithm that naturally allows agents to scale to longer contexts and generalize to more difficult problems via divide-and-conquer. \method~provides a method to train models to best take advantage of such recursive inference, teaching agents \emph{when} and \emph{how} to delegate and communicate. We find that recursive agents trained in this way enjoy better training efficiency, can scale to tasks that go beyond the model's context window, generalize to tasks much harder than the ones the agent was trained on, and can enjoy reduced wall-clock time compared to single-agent systems.\footnote{Webpage and code: \url{https://apga.github.io/RAO}}
\end{abstract}

\vspace{-0.3cm}
\section{Introduction}
\vspace{-0.2cm}

Large language model (LLM) agents are increasingly deployed on real-world tasks such as software engineering, research assistance, and computer use. As these systems improve, the tasks users expect them to solve are also becoming harder: they involve longer horizons, larger effective working memory, and substantial exploration and backtracking. Many such tasks are naturally amenable to divide-and-conquer. For instance, a large software change can be decomposed into code investigation and editing subproblems; a research task can be split into retrieval, synthesis, and verification stages; a long document or log corpus can be partitioned into manageable pieces and processed in parallel.

This has led to interest in \emph{recursive} agent systems~\citep{anthropic_subagents, codex_subagents, context_folding, recursive_language_models}, in which an agent can spawn sub-agents to solve subtasks with fresh contexts. Recursive execution offers several advantages. \textbf{First}, it expands effective working memory, since each sub-agent receives a fresh context window rather than inheriting the full history of its parent. \textbf{Second}, it enables divide-and-conquer by breaking a large task into smaller, more tractable subproblems. \textbf{Third}, when subproblems can be solved independently, recursion can exploit concurrency and reduce wall-clock latency for accomplishing a task. Together, these properties make recursion a compelling inference-time scaling primitive for agents.

However, most existing recursive and multi-agent systems treat recursion purely as an inference-time scaffold wrapped around a pretrained model~\citep{autogen, adapt, redel, asyncsoftware}. 
The model itself is typically not trained to decide when delegation is useful, how to formulate effective subtasks, how to communicate information across levels of the execution tree, or how to combine sub-agent outputs into a final solution. If recursive execution is going to be a core test-time primitive, then the policy should be trained to use it well. As such, prior work shows training models with inference-time scaffolds in the loop does make them more amenable to deployment with inference-time scaffolds~\citep{wu2026reasoning,qednano2026,kimi}. This raises a central question: \emph{how should we train a model to exploit recursive inference effectively?}

In this work, we introduce \textit{\methodlong~(\method)}, a reinforcement learning approach for end-to-end training of recursive agents. \method{} trains a single LLM policy that is instantiated at every node of a recursively-generated execution tree. As a result, the same model must learn both how to solve the task assigned to it and how to generate useful delegated subtasks for spawned copies of itself. Importantly, \method{} is not restricted to a fixed hierarchy or hand-designed orchestration scheme: it supports dynamically generated recursive execution trees and optimizes all decisions in the tree jointly. Training recursive agents may also improve learning efficiency, not just inference-time behavior. Recursive execution naturally generates related tasks at multiple levels of difficulty, yielding structured intermediate supervision and an implicit curriculum over simpler subproblems. Thus, the recursive structure used at inference time can also provide a useful training signal for generally improving model capability. This motivates a second question that we tackle: \emph{how can we leverage the recursive structure of inference to better train agents?}

We evaluate \method{} on three benchmarks, including deep research, long document processing and \textsc{TextCraft-Synth}, a controlled synthetic environment we introduce, inspired by Minecraft-style crafting tasks, in which we can vary task horizon, difficulty, and recursive structure. We find that training for recursive execution yields several benefits. Recursive agents trained with \method{} solve tasks that exceed the base model's context window, generalize to substantially harder tasks than those seen during training by leveraging recursive delegation (up to 10 levels deep), and when tasks can be decomposed into parallel subproblems, can reduce wall-clock execution time relative to non-recursive baselines (up to 2.5$\times$). In addition, \method{} improves training efficiency over single-agent training by exploiting recursive decomposition during learning.

\begin{figure}[t]
    \centering
    \includegraphics[width=\linewidth]{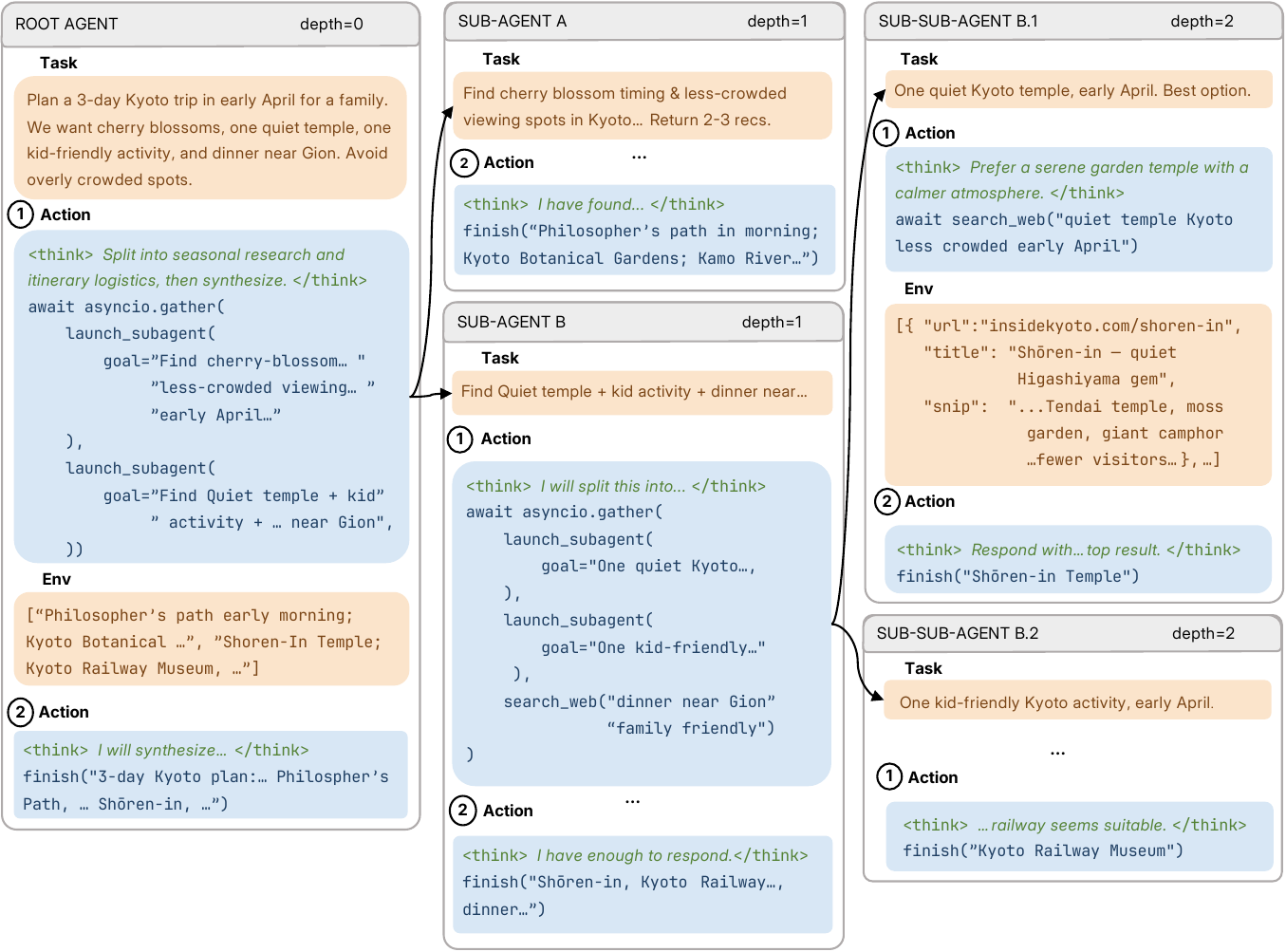}
   \caption{Example of recursive agent inference on a deep research or travel-planning task where agents can decide to spawn and delegate sub-problems to sub-agents. This produces a dynamically structured and recursive execution tree where each node corresponds to one agent instance attempting to solve an assigned task. \method~optimizes a single policy to act across all levels tree, teaching it to best take advantage of recursive inference. }
    \label{fig:recursive-inference}
\end{figure}

\section{Recursive Agents: Inference and Training}

We study \emph{recursive agents}: agents that can delegate sub-tasks to new instances of themselves during execution. A rollout no longer resembles a single flat trajectory, but a dynamically generated tree of trajectories. While many prior works and deployed systems have explored recursive inference (Section~\ref{sec:related-work}), this work focuses on how to best train a model to leverage recursion: it must learn how to solve assigned tasks directly, when delegation is beneficial, how to formulate sub-tasks for child agents, and how to combine their outputs into a final solution. Unlike approaches based on fixed hierarchies or hand-specified orchestration, our setting allows the policy to generate execution trees dynamically, with arbitrary branching patterns up to a depth limit, and RAO trains all agent instances in the tree jointly.

\subsection{Recursive Agent Inference}

\textbf{Formulation.} A task $X$ is a problem specification given to an agent. When the policy $\pi_\theta$ is run on $X$, it produces a trajectory $\tau_X$, consisting of the sequence of actions and observations used to solve the task. During this rollout, the agent may decide to spawn zero or more child agents on delegated sub-tasks $X_1,\dots,X_n$, where both the number and the content of these sub-tasks are chosen by the policy itself. This process induces a rooted \emph{execution tree} $\mathcal{T}$: each node corresponds to one agent instance attempting one assigned task, the root corresponds to the original task, and the children of a node correspond to the delegated sub-tasks generated by that agent. Fig.~\ref{fig:recursive-inference} shows an example of a recursive agent trajectory on a travel-planning task. In particular, note that sub-agents can also recursively spawn their own child agents.

\textbf{Implementation.} We instantiate recursive agents as an extension of an agent that interleaves natural-language reasoning with code execution in a Python REPL (read-eval-print-loop) similar to \cite{codeact}. This interface gives the model access to standard programming constructs such as variables, control flow, string manipulation, and asynchronous execution, allowing it to solve tasks by writing short programs and executing them. To support recursion, we extend the agent's action space by exposing an asynchronous function \[ \texttt{async launch\_subagent}(\texttt{goal}, \ldots) \rightarrow \texttt{Any}\footnote{The exact parameters and return type that the \texttt{launch\_subagent} function takes can be tailored to tasks; Appendix~\ref{app:action-space} provides examples.}, \] which launches a new instance of the same policy on a delegated sub-task and returns the child agent's output to the parent. Because the return type is unrestricted, parents can delegate a wide range of subproblems and request outputs in whatever format is most useful for downstream computation, including structured objects rather than just strings. 

The parent interacts with child agents through ordinary Python control flow. It can launch children sequentially when later sub-tasks depend on earlier results, or concurrently using standard Python libraries like \texttt{asyncio} when sub-tasks are independent. Returned objects can be stored in variables, transformed, combined, or passed into later sub-agent calls. This makes recursive decomposition highly flexible: the policy decides when to delegate, what sub-task specification to provide, what type of output to request, whether to run child agents serially or in parallel, and how to aggregate their results. To ensure bounded computation, we impose explicit limits on recursion depth and on the number of environment steps available to each agent instance. We note that our agent harness implementation resembles that of Recursive Language Models (RLMs) by \cite{recursive_language_models}, with the main differences being that we demonstrate positive results for recursion depths more than one, which have proven elusive in previous works \citep{rlmsoverthink}, and we support concurrent sub-agent execution by implementing the delegation primitive as an asynchronous function. Furthermore, as we discuss next, the focus of this work is on training models to leverage recursion effectively at inference-time.

\subsection{\methodlong~(\method)}\label{sec:rao}

A rollout on a root task produces an execution tree whose nodes correspond to recursively instantiated copies of the same policy. RAO trains all nodes in this tree jointly. At a high level, RAO combines a local reward defined at each node with policy optimization over recursively generated execution trees.

\subsubsection{Local Node Reward}
\begin{wrapfigure}{r}{0.4\textwidth}
\vspace{-1.0em}
\centering
\includegraphics[width=0.4\textwidth]{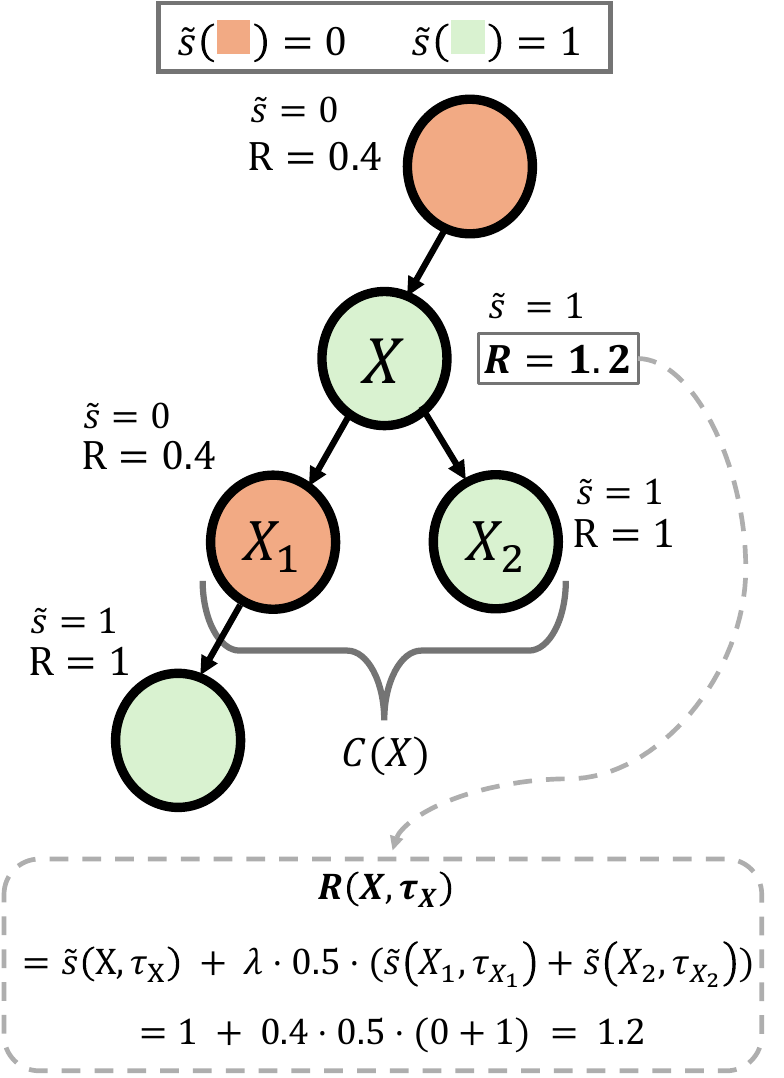}
\vspace{-0.8em}
\caption{
\method{} Reward Design. Each node receives a local reward from its own success and a delegation bonus from the success rate of its children. The example uses $\lambda=0.4$.
}
\label{fig:rao-credit-assignment}
\vspace{-1.0em}
\end{wrapfigure}

A recursive agent should learn both to solve its assigned task and, when useful, to delegate productively. Ideally, we can take advantage of node-local credit assignment: a sub-agent should receive signal from whether its own assigned task and its delegated sub-tasks were solved successfully, rather than relying on the root agent's outcome-level reward. Such local signals can substantially improve credit assignment for decomposition, especially early in training when root-task success may be rare or when the initialized policy does not yet make meaningful use of delegation. At the same time, in many environments fine-grained verification is unavailable or expensive, so direct success signals for intermediate sub-agents may not always be observable.

To accommodate both cases, we distinguish between the underlying notion of task success and the supervisory signal actually used during training each sub-agent. For a node corresponding to task $X$, let $\tau_X$ denote the trajectory generated while solving $X$, and let $C(X)$ denote the set of its immediate children. We use $\tilde{s}(X,\tau_X) \in [0,1]$
for the success signal used in training for node $X$. Depending on the environment, $\tilde{s}(X,\tau_X)$ may be instantiated in different ways: it may come from exact sub-task verification, from a learned or LLM-based judge, or, when no node-local supervision is available, from a proxy such as the success of the root task. 
RAO's credit assignment becomes more faithful as $\tilde{s}(X,\tau_X)$ more accurately reflects the success of the corresponding node-local task $X$.

We define the reward for a given node as:
\begin{equation}
\label{eq:rao_reward}
R(X,\tau_X)
=
\underbrace{\tilde{s}(X,\tau_X)}_{\text{success(X) / proxy}}
+
\underbrace{\lambda \cdot
\frac{1}{|C(X)|}
\sum_{c \in C(X)} \tilde{s}(c,\tau_c)}_{\text{delegation bonus}},
\end{equation}
with the convention that the second term is zero when $|C(X)| = 0$. Here $\lambda \ge 0$ controls the strength of the delegation bonus.

The first term rewards the agent for solving its own assigned task according to the available supervisory signal. The second term rewards it when the sub-tasks it creates are successfully completed by its children. Using the \emph{success rate} of immediate children, rather than the raw number of successful children, avoids directly rewarding the policy for spawning more children purely to collect additional bonus. This makes the delegation bonus better aligned with the quality of delegation rather than the quantity of delegation.

The reward is local: every node, whether root, internal, or leaf, is scored using the same rule based only on its own task signal and the signals of its immediate children. Setting $\lambda = 0$ recovers a purely local-success-based reward. In practice, the delegation bonus is most useful in regimes where the initial policy under-utilizes delegation and needs additional signal to learn when and how to decompose. When the policy already delegates sufficiently at initialization, we can simply set $\lambda=0$ and train with the node-level success signal alone.  Fig.~\ref{fig:rao-credit-assignment} shows an example of reward computation with $\lambda=0.4$.

\subsubsection{Policy Optimization Objective}

Recursive execution induces a family of related task distributions across depths. Let $\mathcal{D}_0$ denote the root task distribution, and let $\mathcal{D}_d(\theta)$ for $d \ge 1$ denote the distribution of depth-$d$ sub-tasks generated by recursively applying the current policy. Training therefore optimizes a shared policy over both root tasks and policy-generated descendants:
\begin{equation}
\label{eq:rao_objective}
J(\theta)
=
\frac{1}{D + 1}\sum_{d=0}^{D}
\mathbb{E}_{X \sim \mathcal{D}_d(\theta)}
\left[
\mathbb{E}_{\tau_X \sim \pi_\theta(\cdot \mid X)}
\big[
R(X,\tau_X)
\big]
\right].
\end{equation}
This view helps explain why recursive training can be effective: the same parameters are trained across a hierarchy of related tasks, and the induced sub-tasks are often simpler or more structured than the original root task, generating a natural curriculum.

For each root task, we sample $G$ independent recursive rollout trees $\{\mathcal{T}^{(g)}\}_{g=1}^G$. Let $R_{\mathrm{root}}^{(g)}$ denote the reward of the root node in rollout $g$. For any trajectory $\tau$ belonging to rollout $g$, we define its advantage using a leave-one-out baseline over root rewards:
\begin{equation}
\label{eq:rao_advantage}
A(\tau^{(g)})
=
R(\tau^{(g)}) - b_{-g},
\qquad
b_{-g}
=
\frac{1}{G-1}
\sum_{g' \neq g}
R_{\mathrm{root}}^{(g')}.
\end{equation}
We use the same root-group baseline for all trajectories within a rollout tree, including child trajectories. Although this may not be the lowest-variance baseline, especially when each sub-task is different, it places all nodes generated under the same root task on a common reference scale and is practical by avoiding the need for a critic or constructing separate comparison groups over policy-generated sub-tasks, which generally differ across rollouts. Furthermore, a standard leave-one-out argument shows that this baseline is unbiased and hence sound; we defer the proof to Appendix~\ref{app:lemma-unbiased}.

A practical issue is that recursive rollouts may contain very different numbers of trajectories at different depths. In some domains, sub-agent trajectories can greatly outnumber root trajectories, causing optimization to become dominated by deeper parts of the execution tree if all trajectories are simply pooled together. To mitigate this effect, we use \emph{depth-level inverse-frequency weighting}, which downweights trajectories from depths that are overrepresented in the batch while preserving the overall scale of the update.

Let $\mathcal{B}_d$ denote the set of all depth-$d$ trajectories in the batch,
and let $N_d = |\mathcal{B}_d|$ and $N = \sum_{d=0}^{D} N_d$.
We assign each trajectory at depth $d$ the inverse-frequency weight
\begin{equation}
\label{eq:rao_depth_weight}
w_d
=
\frac{N}{(D+1)N_d}.
\end{equation}
Since $N_d w_d = N/(D+1)$, each depth receives the same total weight,
while $\sum_{d=0}^{D} N_d w_d = N$, preserving the scale of the usual
batch average. The corresponding estimator is
\begin{equation}
\label{eq:rao_gradient_weighted}
\widehat{\nabla J}(\theta)
=
\frac{1}{N}
\sum_{d=0}^{D}
\sum_{\tau \in \mathcal{B}_d}
w_d \, A(\tau)\,\nabla_\theta \log \pi_\theta(\tau)
=
\frac{1}{D+1}
\sum_{d=0}^{D}
\frac{1}{N_d}
\sum_{\tau \in \mathcal{B}_d}
A(\tau)\,\nabla_\theta \log \pi_\theta(\tau).
\end{equation}

Intuitively, this assigns smaller weight to trajectories from depths that appear more frequently in the batch, reducing the tendency of heavily populated levels of the tree to dominate learning, while the normalization keeps the overall update magnitude approximately unchanged. %

\begin{tcolorbox}[
    enhanced,
    colframe=black,
    colback=evalAll!35,
    boxrule=0.9pt,
    arc=3pt,
    left=8pt,
    right=8pt,
    top=7pt,
    bottom=7pt,
    before skip=8pt,
    after skip=8pt,
    title=\textbf{Summary: RAO},
    coltitle=white,
    colbacktitle=black,
    colframe=black,
    fonttitle=\bfseries,
    titlerule=0pt,
    toptitle=3pt,
    bottomtitle=3pt,
    lefttitle=8pt,
    righttitle=8pt
]
\begin{enumerate}[
    leftmargin=2em,
]
    \item Trains a \textbf{single, shared policy} to act across recursive rollouts with dynamically generated execution trees.
    \item Provides \textbf{dense credit assignment} through a local reward for each node in the recursive execution tree (Eq.~\ref{eq:rao_reward}).
    \item Computes advantages by comparing each node's local reward to a \textbf{leave-one-out baseline computed from root rollout rewards} (Eq.~\ref{eq:rao_advantage}).
    \item Optimizes a \textbf{weighted, multi-task objective} over tasks sampled from different depths of recursive rollouts, yielding a \textbf{self-induced curriculum} (Eqs.~\ref{eq:rao_objective} and \ref{eq:rao_gradient_weighted}).
\end{enumerate}
\end{tcolorbox}

\section{Experiments}
\label{sec:experiments}

We evaluate \method{} on three benchmarks and give an overview below. The full action space for each of the benchmarks is provided in Appendix~\ref{app:action-space}, while additional experiment details and hyperparameters are provided in Appendix~\ref{app:exp-details}.

\subsection{TextCraft-Synth}
\begin{wrapfigure}{r}{0.28\textwidth}
    \vspace{-2em}
    \centering
    \includegraphics[width=0.28\textwidth]{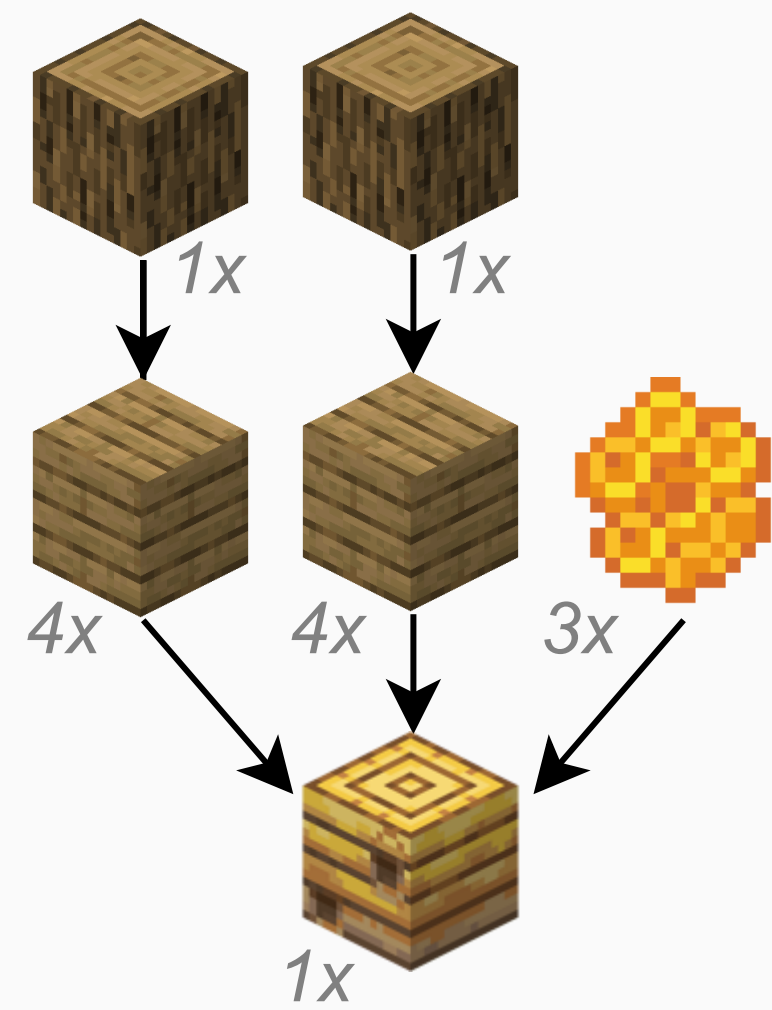}
    \vspace{-0.8em}
    \caption{Figure adapted from~\cite{adapt} visualizing a \textsc{TextCraft} crafting tree. In order to craft the target beehive, the agent must first craft oak planks from oak logs.}
    \label{fig:textcraft}
    \vspace{-1.0em}
\end{wrapfigure}
To study the properties of recursive-agent training in a controlled setting, we introduce \textcraft, inspired by the \textsc{TextCraft} benchmark from \citet{adapt}. In \textsc{TextCraft}, an agent is given an initial inventory and a target item to craft using Minecraft recipes. \textcraft~is a more challenging testbed: we expand the action space with additional APIs, such as recipe search, and replace Minecraft recipes with synthetically generated items and dependencies. This synthetic construction lets us scale the benchmark to arbitrarily deep tasks. Crafting is naturally compositional and recursive, since an item required for the target may itself need to be crafted through multiple intermediate steps. We include tasks of varying difficulty and length by controlling the depth of the underlying crafting tree: \textit{Easy} tasks have depth 2--3, \textit{Medium} tasks depth 4--6, and \textit{Hard} tasks depth 7--9.

For \textcraft, we study two training settings. In the \emph{constrained-context} setting, we limit the model to an 8K context window during both training and inference to test whether RAO can exploit the effectively longer horizon enabled by recursive decomposition. In the \emph{unconstrained} setting, we use a 40K context window during training and 256K during inference. In both settings, we train only on medium-difficulty tasks and evaluate generalization to easy and hard tasks at test time. We use Qwen-3-4B-Instruct-2507~\citep{qwen3} for all \textcraft~experiments. Notably, 40K tokens is already more than sufficient for easy and medium tasks, and 256K is sufficient for all tasks; so in the unconstrained setting the single-agent baseline is not limited by context length. We do not include a delegation bonus ($\lambda=0$) since the agent obtained at initialization from the base model is already able to solve at least a few tasks by employing recursion to some extent. We limit maximum recursion depth to 6 during training, but allow up to a depth of 12 for evaluation. 

\subsection{Oolong-Real}

We also evaluate on \oolong~\citep{oolong}, a long-context benchmark that requires aggregating information from very long Dungeons \& Dragons transcripts. We chose this setting because long-context reasoning is central to many practical agent tasks: for example, a software agent may need to identify the cause of a failure in a very long log file, while a research agent may need to synthesize information across many documents.

For this benchmark, we train Qwen3-VL-30B-A3B-Instruct~\citep{qwen3} with a 32K context limit. This limit is imposed by our training infrastructure: we use the Tinker service~\citep{tinker}, which supports at most 32K tokens for this model. Since \oolong~instances require processing documents of at least 55K tokens, a single-agent baseline cannot read the full input and must instead rely on heuristics such as regex, string matching, or selectively printing portions of the document. A recursive agent, by contrast, can process the entire input by spawning sub-agents over chunks, each with a fresh context window. This therefore provides a realistic test of whether recursive agents can overcome practical context constraints through decomposition. To speed up training, we only train on input contexts with lengths $<=240$K characters ($\sim$60K tokens), but evaluate on contexts up to $220$K tokens. Both training and inference were constrained to a maximum recursion depth of 2 (0-indexed; 3 levels including the root agent). In contrast to the 32K training context window limit, for inference during evaluation, we use a context window of 256K tokens.

For the root-task reward, we follow the \oolong~scoring procedure, which uses partial credit for numerical answers, exact match for string answers, and overlap-based partial credit for list answers. For sub-agent success signal, unlike \textcraft~we do not have perfect verifiers, and so instead rely on an LLM-judge using GPT-5-mini~\citep{gpt-5-mini} (prompt provided in Appendix~\ref{app:oolong-prompts}). To help warm start training, we also use a delegation bonus with $\lambda=0.4$. Similar to \citet{recursive_language_models}, we do not place the entire input directly into the agent prompt. Instead, we pre-populate the agent's Python interpreter with a \texttt{context} string that the agent can interact with programmatically. The recursive agent can then pass smaller chunks to sub-agents via \texttt{launch\_subagent(goal: str, context: str)}.

\subsection{DeepDive}\label{sec:deepdive}

Finally, we benchmark \method~on the deep research domain using the \deepdive{} dataset~\citep{deepdive}. The dataset contains challenging QA pairs constructed by performing controlled walks over knowledge graphs and generating questions that require multi-hop, iterative web searches and synthesis over information scattered across the web to answer. In particular, answering \deepdive{} questions requires breaking down the question into a chain of sequentially dependent sub-queries. A representative example is provided below:
\begin{quote}
	
\small\textbf{Question:} A "Historic State" in Southeast Asia, known for an early dynasty, employed high administrators. One such advisor (active late 15th c.) notably forgave a royal scion who harmed his child. This scion became the state's last sovereign before its early 16th c. European conquest. This final sovereign is linked to a legend: a mystical woman on a high peak made demands of his ruling predecessor, including the sovereign's own youthful life-essence, for marriage. This legend became a film. What year was this film, detailing these conditions, released? \textbf{Answer:} 1962.
\end{quote}

We train Qwen-3-4B-Instruct-2507 on this task with a 40K context limit while training and a 256K context window during evaluation. Both training and evaluation use a maximum recursion depth of 4. We set $\lambda=0$, like for \textcraft{} and similar to \oolong{}, we use an LLM-judge with GPT-5-mini for scoring sub-agent trajectories  (prompt provided in Appendix~\ref{app:deepdive-prompts}).

Beyond the delegation action, the agent has access to functions for searching the web and reading URL content\footnote{We leverage \href{https://www.tavily.com/}{Tavily} APIs to implement these web related actions.}.

\begin{table*}[t]
\centering
\caption{\textcraft~results across evaluation difficulties. Success rate (SR) is computed over the evaluation set. Steps and wall-clock time are computed over the intersection of tasks successfully solved by both methods at each difficulty.}
\label{tab:textcraft-synth-side-by-side}
\footnotesize
\setlength{\tabcolsep}{4pt}

\begin{minipage}[t]{0.46\textwidth}
\centering
\textbf{(a) Context Window: 8K train, 8K eval}\vspace{2pt}

\begin{NiceTabular}{llccc}[cell-space-top-limit=2pt,cell-space-bottom-limit=2pt]
\toprule
Difficulty & Method & SR & Steps & Time (s) \\
\midrule
\Block[fill=evalAll!35]{2-1}{\centering All}
  & Single    & 0.24          & 16 & 7.1 \\
  & Recursive & \textbf{0.95} & 33 & 9.9 \\
\cmidrule(lr){1-5}

\Block[fill=evalEasy!35]{2-1}{\centering Easy}
  & Single    & 0.55          & 12 & 5.2 \\
  & Recursive & \textbf{1.0}  & 23 & 8.0 \\
\cmidrule(lr){1-5}

\Block[fill=evalMedium!35]{2-1}{\centering Medium}
  & Single    & 0.17          & 25 & 11.1 \\
  & Recursive & \textbf{0.96} & 52 & 13.6 \\
\cmidrule(lr){1-5}

\Block[fill=evalHard!35]{2-1}{\centering Hard}
  & Single    & 0.0           & -- & -- \\
  & Recursive & \textbf{0.88} & -- & -- \\
\bottomrule
\end{NiceTabular}
\end{minipage}
\hfill
\begin{minipage}[t]{0.46\textwidth}
\centering
\textbf{(b) Context Window: 40K train, 256K eval}\vspace{2pt}

\begin{NiceTabular}{llccc}[cell-space-top-limit=2pt,cell-space-bottom-limit=2pt]
\toprule
Difficulty & Method & SR & Steps & Time (s) \\
\midrule
\Block[fill=evalAll!35]{2-1}{\centering All}
  & Single    & 0.73          & 54  & 35.7 \\
  & Recursive & \textbf{0.96} & 115 & 19.8 \\
\cmidrule(lr){1-5}

\Block[fill=evalEasy!35]{2-1}{\centering Easy}
  & Single    & 0.97          & 11  & 6.6 \\
  & Recursive & \textbf{1.0}  & 21  & 8.8 \\
\cmidrule(lr){1-5}

\Block[fill=evalMedium!35]{2-1}{\centering Medium}
  & Single    & 0.87          & 60  & 38.1 \\
  & Recursive & \textbf{0.98} & 109 & \textbf{20.9} \\
\cmidrule(lr){1-5}

\Block[fill=evalHard!35]{2-1}{\centering Hard}
  & Single    & 0.20          & 252 & 180.0 \\
  & Recursive & \textbf{0.88} & 694 & \textbf{73.3} \\
\bottomrule
\end{NiceTabular}
\end{minipage}
\end{table*}

\begin{figure}[t]
    \centering
    \includegraphics[width=0.85\linewidth]{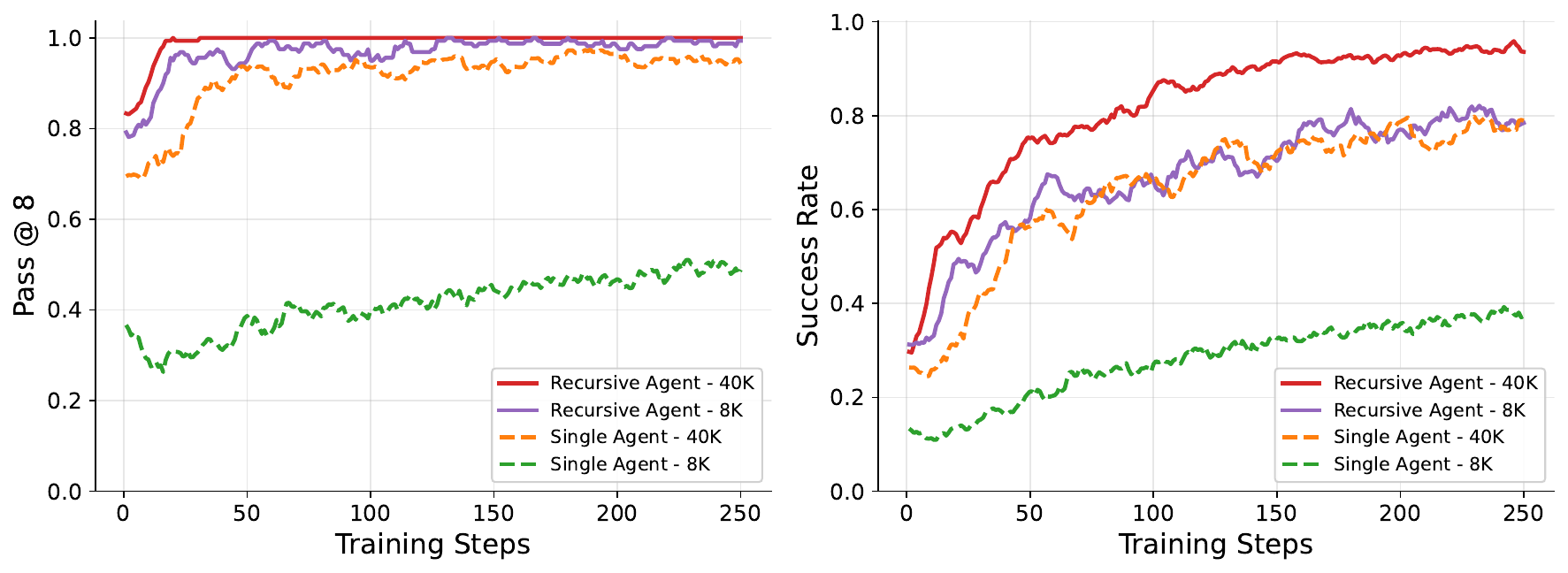}
   \caption{\textcraft~training curves (moving average; window size 10).}
    \label{fig:textcraft-curves}
\end{figure}

\begin{table*}[t]
\centering
\caption{\textsc{Oolong-real} results across context lengths from 55K-175K. Average reward is reported on a sample of 650 bucketed into 55K-175K average input context lengths. Steps and wall-clock time are reported on the common non-zero-score intersection.}
\label{tab:oolong-real-results}
\footnotesize
\setlength{\tabcolsep}{4pt}

\begin{NiceTabular}{lcccccc}[cell-space-top-limit=2pt,cell-space-bottom-limit=2pt]
\toprule
Method &
\Block[fill=evalAll!35]{1-1}{\centering Avg.} &
\Block[fill=evalEasy!35]{1-1}{\centering 55K} &
\Block[fill=evalMedium!35]{1-1}{\centering 118K} &
\Block[fill=evalHard!35]{1-1}{\centering 175K} &
\Block[fill=black!8]{1-1}{\centering Steps} &
\Block[fill=black!8]{1-1}{\centering Time (s)} \\
\midrule

Single    & 0.203 & 0.351 & 0.183 & 0.129 & \textbf{7.1}  & \textbf{12.6} \\
Recursive & \textbf{0.320} & \textbf{0.454} & \textbf{0.315} & \textbf{0.249} & 61.5 & 175.4 \\
\bottomrule
\end{NiceTabular}
\end{table*}

\begin{figure}[t]
    \centering
    \includegraphics[width=0.8\linewidth]{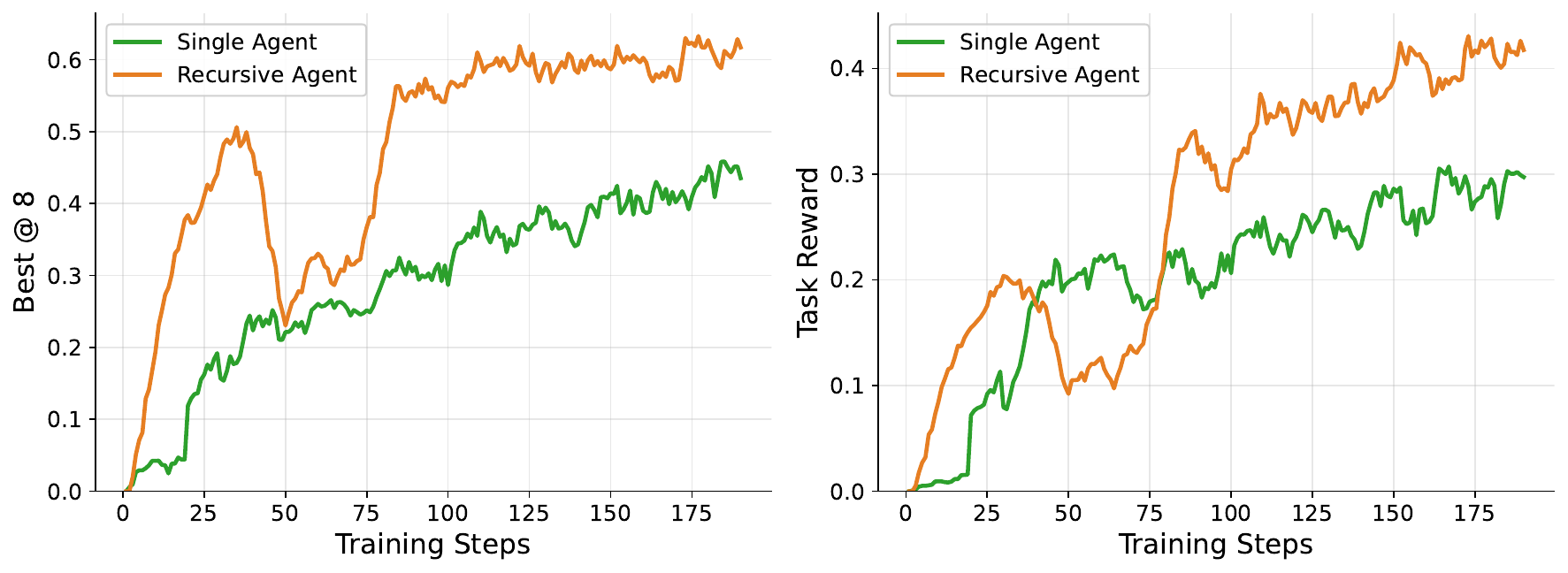}
   \caption{\oolong~training curves (moving average; window size 10).}
    \label{fig:oolong-curves}
\end{figure}

\begin{table*}[t]
\centering

\begin{minipage}[t]{0.58\textwidth}
    \centering
    \vspace{0pt}
    \includegraphics[width=\linewidth]{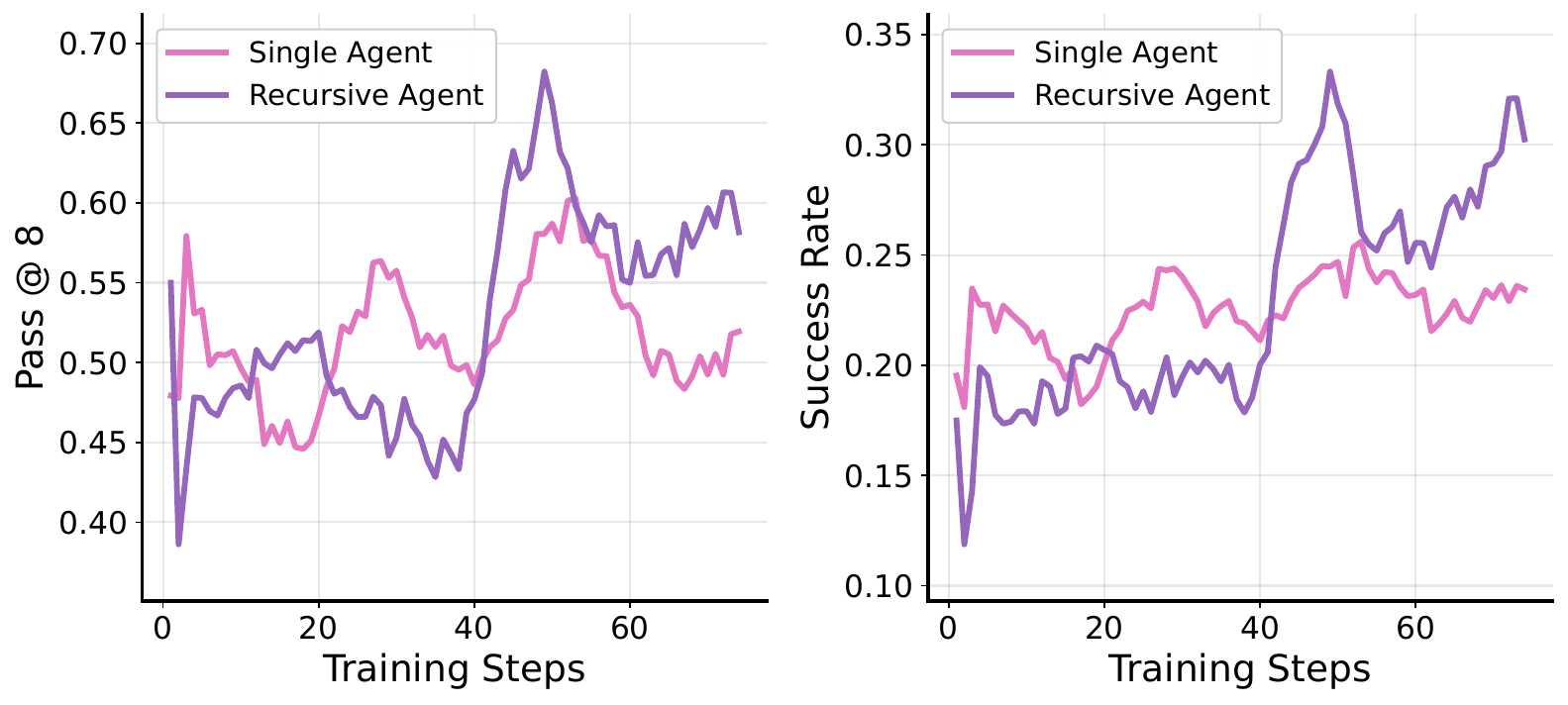}
    \captionof{figure}{\deepdive~training curves (moving average; window size 10)}
    \label{fig:deepdive-curves}
\end{minipage}
\hfill
\begin{minipage}[t]{0.38\textwidth}
    \centering
     \vspace{0.27in}
    \caption{\textsc{DeepDive} evaluation results on 50 held-out tasks. Steps and wall-clock time are reported on tasks solved by both methods.}
    \label{tab:deepdive-results}
    \footnotesize
    \setlength{\tabcolsep}{4pt}

    \begin{NiceTabular}{lccc}[cell-space-top-limit=2pt,cell-space-bottom-limit=2pt]
    \toprule
    Method &
    \Block{1-1}{\centering SR} &
    \Block{1-1}{\centering Steps} &
    \Block{1-1}{\centering Time (s)} \\
    \midrule
    Single    & 0.24 & \textbf{5.2}   & \textbf{13.3} \\
    Recursive & \textbf{0.40} & 121.2 & 233.0 \\
    \bottomrule
    \end{NiceTabular}
\end{minipage}

\end{table*}

\section{Discussion of Experimental Results}\label{sec:discussion}

Across all three benchmarks, recursive agents trained with RAO outperform flat single-agent
baselines. On \textsc{TextCraft-Synth}, recursion improves success rates in both the
constrained-context and unconstrained-context settings, with especially large gains on hard
tasks (Table~\ref{tab:textcraft-synth-side-by-side}). On \textsc{Oolong-Real}, recursive agents achieve
higher average reward across all evaluated context-length buckets despite the 32K training
context limit (Table~\ref{tab:oolong-real-results}). On \textsc{DeepDive}, recursion improves
held-out success rate from 0.24 to 0.40 (Table~\ref{tab:deepdive-results}). Together, these
results suggest that RAO is effective at training agents to exploit recursive execution across
diverse task structures. Below, we discuss specific takeaways in detail.

\textbf{Recursive agents can solve tasks that exceed the model's context window.}
We observe that recursive agents can solve tasks even when limiting training to an 8K context window in \textcraft~in Fig.~\ref{fig:textcraft-curves}. Despite this constraint, training with RAO learns the task effectively and recovers most of the performance of models trained with 40K context. In contrast, the single-agent baseline trained with 8K context lags far behind in both success rate and pass@8. The same pattern appears in Table~\ref{tab:textcraft-synth-side-by-side} under constrained evaluation: the single-agent baseline reaches only 24\% overall success and 0\% on hard tasks, whereas the recursive model with 8K context matches the unconstrained recursive model, achieving 95\% overall success and 88\% on hard tasks.

We observe similar evidence on \oolong. Although these models are trained with a 32K context limit, the tasks require processing at least $\sim$55K tokens. Fig.~\ref{fig:oolong-curves} shows that the recursive agent leverages recursion to effectively expand its usable context and achieves a large gap over the single-agent baseline. We interestingly observe a transient failure mode around training steps 40--80, where the recursive agent briefly learns to print the entire input in the root context, exhausting its context window. It then quickly recovers by relearning the correct strategy: chunking the input and delegating smaller context chunks to sub-agents.

\textbf{Recursive agents enjoy better training efficiency even when not context constrained.}
Even when the model is given a large enough context window to solve the task directly, recursive agents learn substantially faster. On \textcraft~, Fig.~\ref{fig:textcraft-curves} shows that the recursive agent trained with 40K context attains a higher pass@8 much earlier in training and maintains a strong lead in success rate over the single-agent approach. As discussed in Section~\ref{sec:rao}, we attribute this to the structured sub-agent rewards that serve a role akin to dense process rewards and the self-generated curriculum induced by recursive execution. On \deepdive{} (Fig.~\ref{fig:deepdive-curves}), we also observe that the single agent approach learns much more slowly compared to the recursive agent, ending up with a success rate gap of more than 8 percentage points after just 75 steps of training, with the trend suggesting that the gap would further increase with more training. The training performance gap also translates to the \deepdive{} evaluation results: Table~\ref{tab:deepdive-results} shows that the recursive agent beats the single agent by 16 percentage points in success rate.

\paragraph{Recursive agents generalize better to harder tasks.}
All \textcraft~runs are trained only on medium-difficulty problems. Table~\ref{tab:textcraft-synth-side-by-side} shows that while both single-agent and recursive models generalize well to easy tasks, only the recursive agents perform strongly on hard tasks (88\% vs.\ 20\% success rate). We see a similar trend on \oolong: although training uses a filtered prompt set requiring fewer than $\sim$60K tokens, the recursive agent maintains a substantial lead over the single agent even on inputs up to 175K tokens (Table~\ref{tab:oolong-real-results}).

Notably, our 30B recursive model approaches the \oolong~performance of much larger frontier models, including Claude-Sonnet-4 (0.37), o3 (0.37), and GPT-5-mini (0.35) \citep{oolong}. We hypothesize that this stronger test-time scaling comes from teaching the model a divide-and-conquer strategy that naturally transfers to harder problems.

\paragraph{Recursive agents can solve parallelizable tasks faster.}
Beyond improving capability, recursion can also reduce wall-clock time when tasks decompose into independent subtasks that can be executed concurrently. This advantage becomes more pronounced as task difficulty increases. In Table~\ref{tab:textcraft-synth-side-by-side} for \textcraft, the single-agent baseline is slightly faster on easy tasks, but the recursive agent is 1.8$\times$ and 2.5$\times$ faster on medium and hard tasks, respectively, despite taking 1.8$\times$ and 2.8$\times$ more steps.

We do not observe the same advantage on \oolong, where the single-agent baseline is 14$\times$ faster. However, this reflects a qualitative difference in strategy rather than an advantage in solving the task correctly. Because training is limited to 32K context, the single-agent baseline cannot print the full 55K+ token input in context and instead learns to rely almost entirely on fast but inaccurate programmatic heuristics such as regex and string matching. The recursive agent, by contrast, learns the more reliable strategy: printing out and reading the context in chunks via sub-agents. This yields much higher accuracy, but at the cost of longer execution time.

On \deepdive{}, while recursive agents achieve a much higher success rate, they are 18$\times$ slower on tasks solved by both methods (Table~\ref{tab:deepdive-results}). This is because \deepdive{} tasks decompose into sequentially dependent sub-tasks instead of parallelizable ones, as described in Section~\ref{sec:deepdive}. In fact, while 83.9\% of \textcraft{}'s \texttt{launch\_subagent} calls were executed concurrently, only 1.6\% of \deepdive{} delegations were concurrent.

\paragraph{Recursive agents adapt their delegation depth to the task.}
\begin{wrapfigure}{r}{0.38\textwidth}
    \vspace{-2em}
    \centering
    \includegraphics[width=0.38\textwidth]{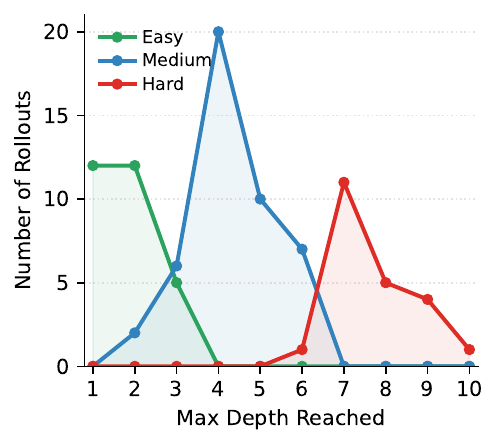}
    \vspace{-0.6cm}
    \caption{Maximum delegation depth on \textcraft.}
    \label{fig:textcraft-depth-plot}
    \vspace{-0.7cm}
\end{wrapfigure}
RAO also teaches agents \emph{when} to delegate and \emph{how much} to delegate. Fig.~\ref{fig:textcraft-depth-plot} plots the maximum depth reached by successful \textcraft~rollouts from the recursive agent. The resulting depth distributions align closely with the task-depth patterns described in Section~\ref{sec:experiments}, suggesting that the agent learns difficulty-appropriate delegation behavior. Moreover, although training caps the maximum depth at 6, the agent still learns to scale to greater depths when solving hard problems.

On \oolong, nearly all successful rollouts have maximum depth 1. This is intuitive: long-context aggregation tasks are best solved by chunking the context once and delegating each chunk to a single layer of children. 
On \deepdive, the average, maximum depth on tasks solved by both the single agent and recursive agent is 2.9, while the depth on tasks uniquely solved by the recursive agent is 4. This suggests that the lift in success for recursive agent on \deepdive{} stems from the fact that it allocates more test-time compute on harder tasks. Overall, RAO learns task-appropriate delegation strategies rather than applying recursion uniformly.

\section{Ablation Experiments}

\subsubsection{Reward Density and Trajectory Weighting}
We perform an ablation study to understand the contribution of reward design and
depth-level inverse-frequency weighting in RAO. For this study, we train different
variants on \textsc{Textcraft-Synth} medium problems for 100 optimization steps.

\noindent\textbf{Reward Density.}
We compare two reward variants. In \emph{Only Root Rewards (Sparse)}, we disable
subagent success rewards and instead propagate the root-agent reward to each
subagent. In \emph{Both Root and Subagent Rewards (Dense)}, both root agents and
subagents receive their own task-specific rewards.

\noindent\textbf{Trajectory Weighting.}
We compare variants with depth-level inverse-frequency weighting enabled
(\emph{Weighted}) or disabled (\emph{Unweighted}).

\begin{wrapfigure}[24]{r}{0.43\textwidth}
    \centering
    \vspace{-3.0em}
    \includegraphics[width=0.43\textwidth]{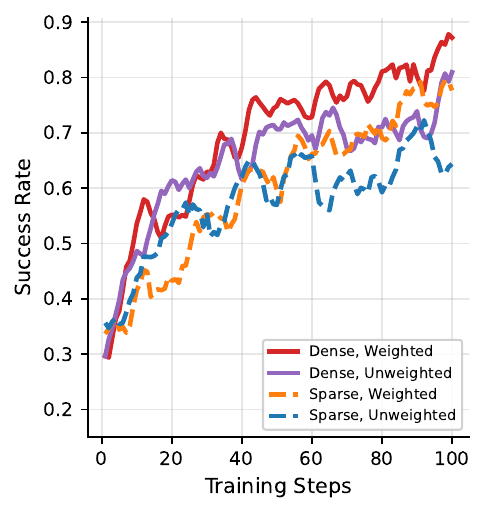}

     \vspace{0.4em}

    \footnotesize
    \setlength{\tabcolsep}{3pt}
    \begin{tabular}{lc}
        \toprule
        \textbf{Method} &
        \textbf{Eval SR} \\
        \midrule
        Dense, Weighted   & 0.89  \\
        Dense, Unweighted & 0.75  \\
        Sparse, Weighted  & 0.79 \\
        Sparse, Unweighted& 0.61 \\
        \bottomrule
    \end{tabular}
        \caption{Ablation of reward density and trajectory weighting on \textsc{TextCraft-Synth}. Training curves above with evaluation SRs below.}
    \label{fig:textcraft-ablation}
\end{wrapfigure}

Fig.~\ref{fig:textcraft-ablation} shows the training curves and evaluation success rates for these experiments.
Dense rewards and trajectory weighting are both important for achieving the best
training efficiency and performance.

\subsubsection{Delegation Bonus}

We also perform a second ablation experiment with \textcraft~to better understand the effects of training with different styles of delegation bonus. In particular, we compare the child success \textit{rate} bonus from Eq.~\ref{eq:rao_reward} with a child success \textit{count} bonus, fixing $\lambda=0.4$. For these runs, we perform 50 optimization steps with the Tinker training backend. Fig~\ref{fig:delegation-ablation} plots the training curves, number of sub-agents launched and the time taken per training step. Using the success count as the bonus incentivizes the agent to reward hack by launching many sub-agents on easy but unnecessary sub-tasks, leading to an uncontrolled increase in both the number of sub-agents launched and the time taken per training step. This further also causes the success rate to tank after step 40. In contrast, using the success rate as the bonus avoids this failure mode and keeps training stable.

\begin{figure}[h]
    \centering
    \includegraphics[width=\linewidth]{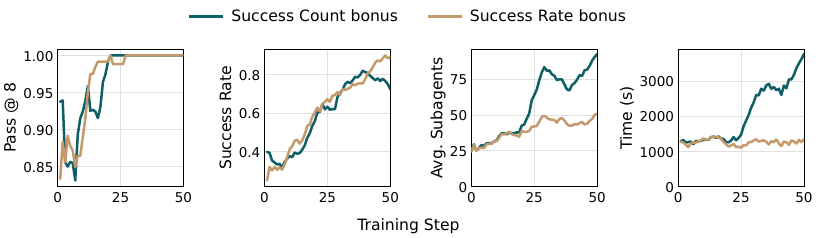}
   \caption{Ablation comparing success count vs. rate delegation bonus on \textsc{TextCraft-Synth}, plotting the success curves, sub-agent count and time per training step.}
    \label{fig:delegation-ablation}
\end{figure}

\section{Related Work}
\label{sec:related-work}

\textbf{Recursive and multi-agent inference scaffolds.}
Recursive and multi-agent scaffolds are increasingly used to scale LLM agents at inference time. 
Deployed coding agents such as Claude Code and Codex expose subagents as a practical delegation primitive for context isolation, parallel work, and specialized tool use~\citep{anthropic_subagents,codex_subagents}. 
Research systems such as ADaPT, THREAD and ReDel support deeper recursive decomposition, allowing agents to recursively break tasks into simpler subproblems~\citep{adapt,thread,redel}. 
However, these systems primarily treat recursion as an inference-time scaffold around a fixed model, rather than training the model to use recursive execution effectively.

\textbf{Training agents to delegate.}
A smaller line of recent work has begun to train models to use delegation. Our agent scaffold is closest to Recursive Language Models (RLMs)~\citep{recursive_language_models}, which also extend a CodeAct-style Python REPL agent \citep{codeact} with a delegation function. 
However, \cite{recursive_language_models} primarily study delegation as an inference-time primitive, and while they perform a small-scale training experiment, it focuses on distillation via SFT on trajectories collected from a larger model, and only trains on root agent trajectories in a depth-1 recursive system. Furthermore, their experiments use only sequential (and not parallel) sub-agent execution.

\begin{wraptable}{r}{0.55\textwidth}
\centering
\caption{Comparison to recursive agent systems. Here, $D{>}1$ denotes experimentation with recursive execution beyond depth 1, Parallel denotes concurrent subagent execution, Joint denotes training both root agents and subagents, Real denotes evaluation on complex agentic applications.}
\label{tab:related-recursive-agents}
\footnotesize
\setlength{\tabcolsep}{3pt}
\begin{tabular}{lccccc}
\toprule
Method &
$D{>}1$ &
RL &
Parallel &
Joint &
Real \\
\midrule
ADaPT / THREAD      & \cmark & \xmark & \xmark & \xmark & \xmark \\
RLMs                & \cmark & \xmark & \xmark & \xmark & \cmark \\
Claude Code / Codex & \xmark & \qmark      & \cmark & \qmark      & \cmark \\
AsyncThink          & \xmark & \cmark & \cmark & \cmark & \xmark \\
Context-Folding     & \xmark & \cmark & \xmark & \cmark & \cmark \\
Kimi K2.5 (PARL)             & \xmark & \cmark & \cmark & \xmark & \cmark \\
\midrule
RAO                 & \cmark & \cmark & \cmark & \cmark & \cmark \\
\bottomrule
\end{tabular}
\vspace{-1.0em}
\end{wraptable}
Kimi K2.5 trains a depth-1 agent-swarm system with with reinforcement learning and supports parallel subagent execution, but they update only the root/orchestrator while keeping worker subagents fixed to an older checkpoint~\citep{kimi}. 
Context-Folding and AsyncThink train both organizers and workers with reinforcement learning, but remain depth-1 systems focused on restricted recursion patterns or narrower domains such as context management, Countdown, or mathematical reasoning~\citep{context_folding,asyncthink}. 
To our knowledge, RAO is the first to jointly study recursive agent training with all of the following: recursion beyond depth 1, end-to-end RL, asynchronous subagent execution, joint training of root agents and subagents, and evaluation on complex agentic applications.

\textbf{Hierarchical reinforcement learning, process rewards, and automatic curricula.}
RAO is also connected to hierarchical reinforcement learning, which studies temporal abstraction and learned subpolicies through frameworks such as options, MAXQ, Feudal RL, and HIRO~\citep{options,maxq,feudal,hiro}. 
Unlike classical HRL, our subtasks are expressed in natural language and generated online by the same language model policy rather than through fixed action abstractions or separate high- and low-level policies. 
Our use of node-local supervision also relates to process rewards and step-level verification~\citep{lets_verify}. 
Finally, recursively generated subtasks induce a curriculum over related task distributions, connecting RAO to automatic curriculum learning methods such as AMIGo~\citep{amigo}. 
In RAO, however, the curriculum is generated cooperatively as part of solving an overall task, rather than adversarially by a separate teacher proposing goals for a student. 
RAO can therefore be interpreted as combining natural-language temporal abstraction, local credit assignment over recursive trees, and shared-parameter learning across a policy-induced hierarchy of task distributions.

\vspace{-0.5em}
\section{Conclusion}
\vspace{-0.5em}

We introduced \methodlong{} (\method), a reinforcement learning approach for training language-model agents to use recursive delegation as a learned inference-time primitive. Rather than wrapping a fixed model in a hand-designed recursive scaffold, \method{} trains a single shared policy across all nodes of a dynamically generated execution tree. The resulting agent learns not only to solve assigned tasks, but also to decide when to delegate, how to specify useful sub-tasks, how to coordinate recursive copies of itself, and how to aggregate their outputs. Across three domains, \method{} consistently unlocks capabilities that are difficult to obtain from flat single-agent execution alone, including better success rates, exploiting parallelism where possible, generalization to harder problems, and solving tasks over horizons far beyond the model's context window. 

\method{} also raises several immediate directions for scaling recursive agents beyond the settings studied here. With agents that can allocate computation through delegation, evaluation should move beyond final success alone and measure how efficiently agents use inference-time compute. While \method{} trains a single shared model across all nodes of the recursive tree, practical systems may benefit from heterogeneous recursion, where stronger models supervise, verify, or synthesize the work of smaller specialist sub-agents, or vice-versa. Finally, because our experiments train recursion separately within each domain, an important next step is to train generalist recursive agents that transfer delegation strategies across domains. This also raises a related question about task heterogeneity within a single recursive rollout: our experiments primarily focus on settings where sub-tasks resemble smaller instances of the parent task, but many real-world tasks require delegation to qualitatively different subproblems, such as retrieval, verification, implementation, debugging, or synthesis. Understanding which environments best teach decomposition, coordination, and long-horizon planning will be crucial for building broadly capable recursive agents, and will require benchmarks that stress-test ultra-long horizons, parallelizable task structure, adaptive compute allocation, and multi-agent coordination.

More broadly, our results point to a simple principle: \textit{inference-time scaffolds should not merely be designed around models; models should be trained to use them.} Recursion is one instance of this principle, but a particularly powerful one: it gives agents fresh context windows, divide-and-conquer structure, adaptive allocation of test-time compute, and opportunities for parallel execution. In this sense, recursion offers a path toward self-organizing agents.
At the same time, as recursive methods and more general multi-agent systems continue to scale, we will need to address a practical training question: how should we design surrogate sampling procedures when full recursive or multi-agent rollouts are too costly to complete inside the RL loop? Answering this will require new training scaffolds, appropriate data mixtures, and objectives that transfer from these surrogates to challenging open-ended tasks. Designing such surrogate procedures is an important direction for future work, with the potential to improve performance while reducing latency.

\section*{Acknowledgements}
We thank Amazon and Thinking Machines Lab for their very generous compute and Tinker training credits support respectively. We are also grateful to Tavily for providing us with credits used for web search APIs in our deep research experiments. Finally, we also thank many others who gave valuable feedback on earlier versions of the methods and experimental results: Pranjal Aggarwal, Zora Wang, Simran Khanuja, Amrith Setlur, Ian Wu, Daniel Fried, Andre He, Andy Liu, Atharva Naik, JY Koh, and Saujas Vaduguru at CMU, as well as Royce Cheng-Yue, Yifei Wang, Rajiv Dhawan and Sameep Tandon at Amazon. Apurva is supported by and is hugely grateful for the Amazon AI PhD Fellowship.

\bibliography{colm2026_conference}

\begin{thebibliography}{30}
\providecommand{\natexlab}[1]{#1}
\providecommand{\url}[1]{\texttt{#1}}
\expandafter\ifx\csname urlstyle\endcsname\relax
  \providecommand{\doi}[1]{doi: #1}\else
  \providecommand{\doi}{doi: \begingroup \urlstyle{rm}\Url}\fi

\bibitem[{Anthropic}(2025)]{anthropic_subagents}
{Anthropic}.
\newblock Subagents in the {Claude Code} {SDK}.
\newblock \url{https://code.claude.com/docs/en/agent-sdk/subagents}, 2025.
\newblock Accessed: 2026-05-06.

\bibitem[Bai et~al.(2025)Bai, Cai, Chen, Chen, Chen, Cheng, Deng, Ding, Gao, Ge, et~al.]{qwen3}
Shuai Bai, Yuxuan Cai, Ruizhe Chen, Keqin Chen, Xionghui Chen, Zesen Cheng, Lianghao Deng, Wei Ding, Chang Gao, Chunjiang Ge, et~al.
\newblock Qwen3-vl technical report.
\newblock \emph{arXiv preprint arXiv:2511.21631}, 2025.

\bibitem[Bertsch et~al.(2025)Bertsch, Pratapa, Mitamura, Neubig, and Gormley]{oolong}
Amanda Bertsch, Adithya Pratapa, Teruko Mitamura, Graham Neubig, and Matthew~R Gormley.
\newblock Oolong: Evaluating long context reasoning and aggregation capabilities.
\newblock \emph{arXiv preprint arXiv:2511.02817}, 2025.

\bibitem[Campero et~al.(2020)Campero, Raileanu, K{\"u}ttler, Tenenbaum, Rockt{\"a}schel, and Grefenstette]{amigo}
Andres Campero, Roberta Raileanu, Heinrich K{\"u}ttler, Joshua~B. Tenenbaum, Tim Rockt{\"a}schel, and Edward Grefenstette.
\newblock Learning with {AMIGo}: Adversarially motivated intrinsic goals, 2020.
\newblock URL \url{https://arxiv.org/abs/2006.12122}.

\bibitem[Chen et~al.(2025)Chen, Li, Gong, Jiang, Fei, Yang, Shan, Yu, Wang, Zhu, et~al.]{cispo}
Aili Chen, Aonian Li, Bangwei Gong, Binyang Jiang, Bo~Fei, Bo~Yang, Boji Shan, Changqing Yu, Chao Wang, Cheng Zhu, et~al.
\newblock Minimax-m1: Scaling test-time compute efficiently with lightning attention.
\newblock \emph{arXiv preprint arXiv:2506.13585}, 2025.

\bibitem[Chi(2025)]{asyncthink}
Zhaoxuan Chi.
\newblock The era of agentic organization: Learning to organize {LLM} reasoning with {AsyncThink}, 2025.
\newblock URL \url{https://arxiv.org/abs/2510.26658}.

\bibitem[Corbitt(2025)]{arte}
Kyle Corbitt.
\newblock {ART{\textperiodcentered}E}: How we built an email research agent that beats o3.
\newblock OpenPipe Blog, April 2025.
\newblock URL \url{https://openpipe.ai/blog/art-e-mail-agent}.
\newblock Accessed: 2026-05-05.

\bibitem[Dayan \& Hinton(1992)Dayan and Hinton]{feudal}
Peter Dayan and Geoffrey~E Hinton.
\newblock Feudal reinforcement learning.
\newblock \emph{Advances in neural information processing systems}, 5, 1992.

\bibitem[Dietterich(2000)]{maxq}
T.~G. Dietterich.
\newblock Hierarchical reinforcement learning with the {MAXQ} value function decomposition.
\newblock \emph{Journal of Artificial Intelligence Research}, 13:\penalty0 227--303, 2000.
\newblock \doi{10.1613/jair.639}.
\newblock URL \url{https://doi.org/10.1613/jair.639}.

\bibitem[Fu et~al.(2025)Fu, Gao, Shen, Zhu, Mei, He, Xu, Wei, Mei, Wang, Yang, Yuan, and Wu]{areal}
Wei Fu, Jiaxuan Gao, Xujie Shen, Chen Zhu, Zhiyu Mei, Chuyi He, Shusheng Xu, Guo Wei, Jun Mei, Jiashu Wang, Tongkai Yang, Binhang Yuan, and Yi~Wu.
\newblock Areal: A large-scale asynchronous reinforcement learning system for language reasoning, 2025.
\newblock URL \url{https://arxiv.org/abs/2505.24298}.

\bibitem[Geng \& Neubig(2026)Geng and Neubig]{asyncsoftware}
Jiayi Geng and Graham Neubig.
\newblock Effective strategies for asynchronous software engineering agents.
\newblock \emph{arXiv preprint arXiv:2603.21489}, 2026.

\bibitem[{Kimi Team} et~al.(2026){Kimi Team}, Bai, Bai, Bao, Cai, Cao, Charles, Che, Chen, Chen, et~al.]{kimi}
{Kimi Team}, Tongtong Bai, Yifan Bai, Yiping Bao, SH~Cai, Yuan Cao, Y~Charles, HS~Che, Cheng Chen, Guanduo Chen, et~al.
\newblock Kimi k2. 5: Visual agentic intelligence.
\newblock \emph{arXiv preprint arXiv:2602.02276}, 2026.

\bibitem[Lightman et~al.(2023)Lightman, Kosaraju, Burda, Edwards, Baker, Lee, Leike, Schulman, Sutskever, and Cobbe]{lets_verify}
Hunter Lightman, Vineet Kosaraju, Yura Burda, Harri Edwards, Bowen Baker, Teddy Lee, Jan Leike, John Schulman, Ilya Sutskever, and Karl Cobbe.
\newblock Let's verify step by step, 2023.
\newblock URL \url{https://arxiv.org/abs/2305.20050}.

\bibitem[LM-Provers et~al.(2026)LM-Provers, Qu, Setlur, Dekoninck, Beeching, Li, Wu, Tunstall, and Kumar]{qednano2026}
LM-Provers, Yuxiao Qu, Amrith Setlur, Jasper Dekoninck, Edward Beeching, Jia Li, Ian Wu, Lewis Tunstall, and Aviral Kumar.
\newblock Qed-nano: Teaching a tiny model to prove hard theorems.
\newblock https://huggingface.co/spaces/lm-provers/qed-nano-blogpost, 2026.
\newblock Blog post.

\bibitem[Lu et~al.(2025)Lu, Hou, Wang, Zhang, Liu, Li, Feng, Tang, and Dong]{deepdive}
Rui Lu, Zhenyu Hou, Zihan Wang, Hanchen Zhang, Xiao Liu, Yujiang Li, Shi Feng, Jie Tang, and Yuxiao Dong.
\newblock Deepdive: Advancing deep search agents with knowledge graphs and multi-turn rl.
\newblock \emph{arXiv preprint arXiv:2509.10446}, 2025.

\bibitem[Nachum et~al.(2018)Nachum, Gu, Lee, and Levine]{hiro}
Ofir Nachum, Shixiang Gu, Honglak Lee, and Sergey Levine.
\newblock Data-efficient hierarchical reinforcement learning, 2018.
\newblock URL \url{https://arxiv.org/abs/1805.08296}.

\bibitem[{OpenAI}(2026)]{codex_subagents}
{OpenAI}.
\newblock Subagents: Use subagents and custom agents in {Codex}.
\newblock \url{https://developers.openai.com/codex/subagents}, 2026.
\newblock Accessed: 2026-05-06.

\bibitem[Prasad et~al.(2024)Prasad, Koller, Hartmann, Clark, Sabharwal, Bansal, and Khot]{adapt}
Archiki Prasad, Alexander Koller, Mareike Hartmann, Peter Clark, Ashish Sabharwal, Mohit Bansal, and Tushar Khot.
\newblock {AD}a{PT}: As-needed decomposition and planning with language models.
\newblock In \emph{Findings of the Association for Computational Linguistics: NAACL 2024}, pp.\  4226--4252. Association for Computational Linguistics, 2024.
\newblock \doi{10.18653/v1/2024.findings-naacl.264}.
\newblock URL \url{https://aclanthology.org/2024.findings-naacl.264/}.

\bibitem[Schroeder et~al.(2024)Schroeder, Morgan, Luo, and Glass]{thread}
Philip Schroeder, Nathaniel Morgan, Hongyin Luo, and James Glass.
\newblock {THREAD}: Thinking deeper with recursive spawning, 2024.
\newblock URL \url{https://arxiv.org/abs/2405.17402}.

\bibitem[Singh et~al.(2025)Singh, Fry, Perelman, Tart, Ganesh, El-Kishky, McLaughlin, Low, Ostrow, Ananthram, et~al.]{gpt-5-mini}
Aaditya Singh, Adam Fry, Adam Perelman, Adam Tart, Adi Ganesh, Ahmed El-Kishky, Aidan McLaughlin, Aiden Low, AJ~Ostrow, Akhila Ananthram, et~al.
\newblock Openai gpt-5 system card.
\newblock \emph{arXiv preprint arXiv:2601.03267}, 2025.

\bibitem[Sun et~al.(2025)Sun, Lu, Ling, Liu, Yao, Yang, and Chen]{context_folding}
Weiwei Sun, Miao Lu, Zhan Ling, Kang Liu, Xuesong Yao, Yiming Yang, and Jiecao Chen.
\newblock Scaling long-horizon {LLM} agent via context-folding, 2025.
\newblock URL \url{https://arxiv.org/abs/2510.11967}.

\bibitem[Sutton et~al.(1999)Sutton, Precup, and Singh]{options}
Richard~S. Sutton, Doina Precup, and Satinder Singh.
\newblock Between {MDPs} and semi-{MDPs}: A framework for temporal abstraction in reinforcement learning.
\newblock \emph{Artificial Intelligence}, 112\penalty0 (1--2):\penalty0 181--211, 1999.
\newblock \doi{10.1016/S0004-3702(99)00052-1}.
\newblock URL \url{https://doi.org/10.1016/S0004-3702(99)00052-1}.

\bibitem[{Thinking Machines Lab}(2025)]{tinker}
{Thinking Machines Lab}.
\newblock Tinker, 2025.
\newblock URL \url{https://thinkingmachines.ai/tinker/}.

\bibitem[Wang(2026)]{rlmsoverthink}
Daren Wang.
\newblock Think, but don't overthink: Reproducing recursive language models.
\newblock \emph{arXiv preprint arXiv:2603.02615}, 2026.

\bibitem[Wang et~al.(2024)Wang, Chen, Yuan, Zhang, Li, Peng, and Ji]{codeact}
Xingyao Wang, Yangyi Chen, Lifan Yuan, Yizhe Zhang, Yunzhu Li, Hao Peng, and Heng Ji.
\newblock Executable code actions elicit better {LLM} agents, 2024.
\newblock URL \url{https://arxiv.org/abs/2402.01030}.

\bibitem[Wu et~al.(2026)Wu, Qu, Setlur, and Kumar]{wu2026reasoning}
Ian Wu, Yuxiao Qu, Amrith Setlur, and Aviral Kumar.
\newblock Reasoning cache: Continual improvement over long horizons via short-horizon rl.
\newblock \emph{arXiv preprint arXiv:2602.03773}, 2026.

\bibitem[Wu et~al.(2023)Wu, Bansal, Zhang, Wu, Li, Zhu, Jiang, Zhang, Zhang, Liu, Awadallah, White, Burger, and Wang]{autogen}
Qingyun Wu, Gagan Bansal, Jieyu Zhang, Yiran Wu, Beibin Li, Erkang Zhu, Li~Jiang, Xiaoyun Zhang, Shaokun Zhang, Jiale Liu, Ahmed~Hassan Awadallah, Ryen~W White, Doug Burger, and Chi Wang.
\newblock Autogen: Enabling next-gen {LLM} applications via multi-agent conversation, 2023.
\newblock URL \url{https://arxiv.org/abs/2308.08155}.

\bibitem[Zhang et~al.(2025)Zhang, Kraska, and Khattab]{recursive_language_models}
Alex~L. Zhang, Tim Kraska, and Omar Khattab.
\newblock Recursive language models, 2025.
\newblock URL \url{https://arxiv.org/abs/2512.24601}.

\bibitem[Zhang et~al.(2024)Zhang, Sun, Chen, Pfister, Zhang, and Ar{\i}k]{coa}
Yusen Zhang, Ruoxi Sun, Yanfei Chen, Tomas Pfister, Rui Zhang, and Sercan~{\"O} Ar{\i}k.
\newblock Chain of agents: Large language models collaborating on long-context tasks.
\newblock \emph{Advances in Neural Information Processing Systems}, 37:\penalty0 132208--132237, 2024.

\bibitem[Zhu et~al.(2024)Zhu, Dugan, and Callison-Burch]{redel}
Andrew Zhu, Liam Dugan, and Chris Callison-Burch.
\newblock Redel: A toolkit for llm-powered recursive multi-agent systems.
\newblock In \emph{Proceedings of the 2024 Conference on Empirical Methods in Natural Language Processing: System Demonstrations}, pp.\  162--171, 2024.

\end{thebibliography}
\bibliographystyle{colm2026_conference}

\appendix
\newpage

\section{Appendix}

\subsection{Additional Baselines}

\begin{figure}[h]
  \centering
  \begin{subfigure}{\textwidth}
    \centering
    \includegraphics[width=\textwidth]{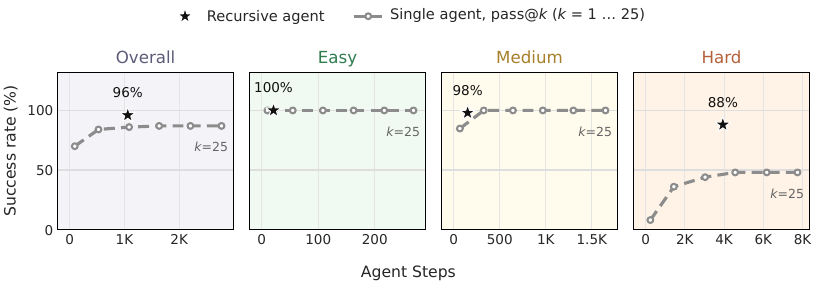}
    \caption{\textsc{TextCraft-Synth}}
    \label{fig:frontier-textcraft}
  \end{subfigure}

  \vspace{0.6em}

  \begin{subfigure}{\textwidth}
    \centering
    \includegraphics[width=\textwidth]{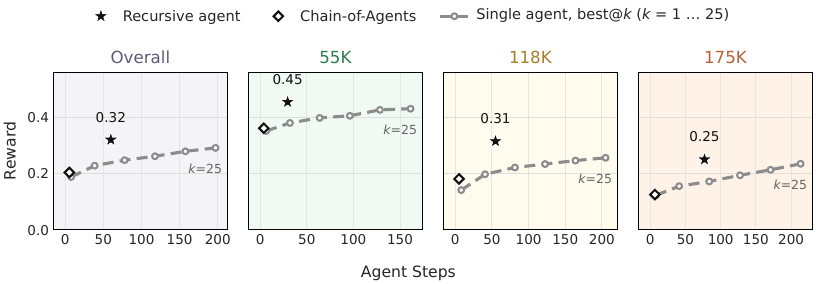}
    \caption{\textsc{Oolong-Real}}
    \label{fig:frontier-oolong}
  \end{subfigure}

  \caption{Comparison of the trained recursive agent with a trained single agent
  under idealized best-of-$k$ and chain-of-agents harnesses. We vary
  $k\in\{1, 5, 10, 15, 20, 25\}$ and plot overall and per-difficulty
  success rates/rewards as a function of the number of agent steps used by each
  configuration. Sampling more does not substitute for delegating: the recursive agent ($\star$) lies above the best-of-$k$
  frontier while using less inference-time compute, despite best-of-25
consuming $2.6\times$--$3.3\times$ more agent steps.}
  \label{fig:additional-baselines}
\end{figure}

To better understand the benefits of training with a recursive harness, we
compare against additional approaches for scaling inference-time computation.
For \textcraft, we evaluate a Pass@$k$ baseline with
$k \in \{1, 5, 10, 15, 20, 25\}$. For \oolong, we similarly evaluate a
Best@$k$ baseline, and additionally include the chain-of-agents baseline
introduced by \cite{coa}. Chain-of-agents is a fixed multi-agent workflow for
long-context processing in which worker agents sequentially process chunks of a
document, each passing a summary (or ``communication unit'') of the accumulated
context to the next worker, until a final agent produces an answer for the full
document. In our experiments, we use a maximum chunk size of 128K characters
($\sim$32K tokens) and a maximum communication-unit size of 32K characters
($\sim$8K tokens).

The best-of-$k$ baselines are deliberately optimistic in both environments.
For \textcraft, Pass@$k$ assumes that crafting actions can be reversed so that
the environment can be reset between attempts. For \oolong, Best@$k$ assumes
access to an oracle ranker that selects the best of the $k$ attempts.
Consequently, these baselines provide an optimistic upper bound on what can be
achieved through repeated sampling of the single-agent policy. For all
baselines, we use the corresponding model RL-trained in the single-agent
setting; for the recursive agent, we use the \method-trained models.

Fig.~\ref{fig:additional-baselines} shows that the trained recursive agent
remains ahead even against these optimistic baselines. On \textcraft, it solves
96\% of tasks in ${\sim}1.1$K agent steps, whereas Pass@$k$ saturates at 87\%
by $k{=}15$, with no further improvement from the additional
steps required to reach $k{=}25$. Easy and Medium tasks are nearly saturated
for both methods, so the gap is concentrated on Hard tasks, where the recursive
agent reaches 88\% success compared with a Pass@25 ceiling of 48\%, while using
roughly half as many agent steps.

On \oolong, the recursive agent achieves an average reward of $0.320$ in 60
steps, exceeding Best@25 ($0.291$ at 197 steps) while using $3.3\times$ less
compute. This ordering holds across every context-length bucket. Chain-of-agents is by far the least expensive configuration, requiring
only $5.5$ steps, but achieves a reward of just $0.203$, close to Best@1
($0.187$).

Taken together, these results show that training explicitly for recursion
outperforms deploying a model trained in the standard single-agent setting
inside either repeated-sampling or chain-of-agents harnesses. The gap relative
to chain-of-agents on \oolong is particularly informative: despite being
designed specifically for long-context processing, chain-of-agents performs
substantially worse than the trained recursive agent. This result underscores
the importance of training models for the inference-time computation structure
in which they will ultimately be deployed, rather than relying on a
single-agent-trained model to generalize to a new harness only at inference
time.

\subsection{When does \method{} work best?}
The main experiments in this work are performed on tasks with non-trivial difficulty and/or length (\textcraft{}, \oolong{} and \deepdive{}). On such tasks, \method{} helps significantly with many advantages over single agent inference and training, as we discuss in detail in Section~\ref{sec:discussion}. 

But what about when tasks are both easy and short-horizon? To test \method{} in this setting, we train on an easier variant of deep research using the \arte{}~\citep{arte} dataset, where an agent needs to search over a user's emails to answer a question. In particular, unlike \deepdive{} tasks which require numerous sequentially dependent searches to find and synthesize information scattered across the internet, \arte{} tasks are much easier, requiring finding only a single relevant email in a user's inbox to answer the question.

In this setting, we find that \method{} still helps the model learn more rapidly in the initial steps of training, but then later the performance equalizes between the single and recursive agents (Fig.~\ref{fig:email-curves}). We also observe this for \textcraft{} in Table~\ref{tab:textcraft-synth-side-by-side}b where a significant gap in performance appears only for medium and higher difficulties and not on easy tasks. This is intuitive: \textit{\method{} helps the most when tasks are either difficult enough to benefit from divide-and-conquer and/or long-horizon enough to require extended context windows.}

\begin{figure}[h]
    \centering
    \includegraphics[width=0.8\linewidth]{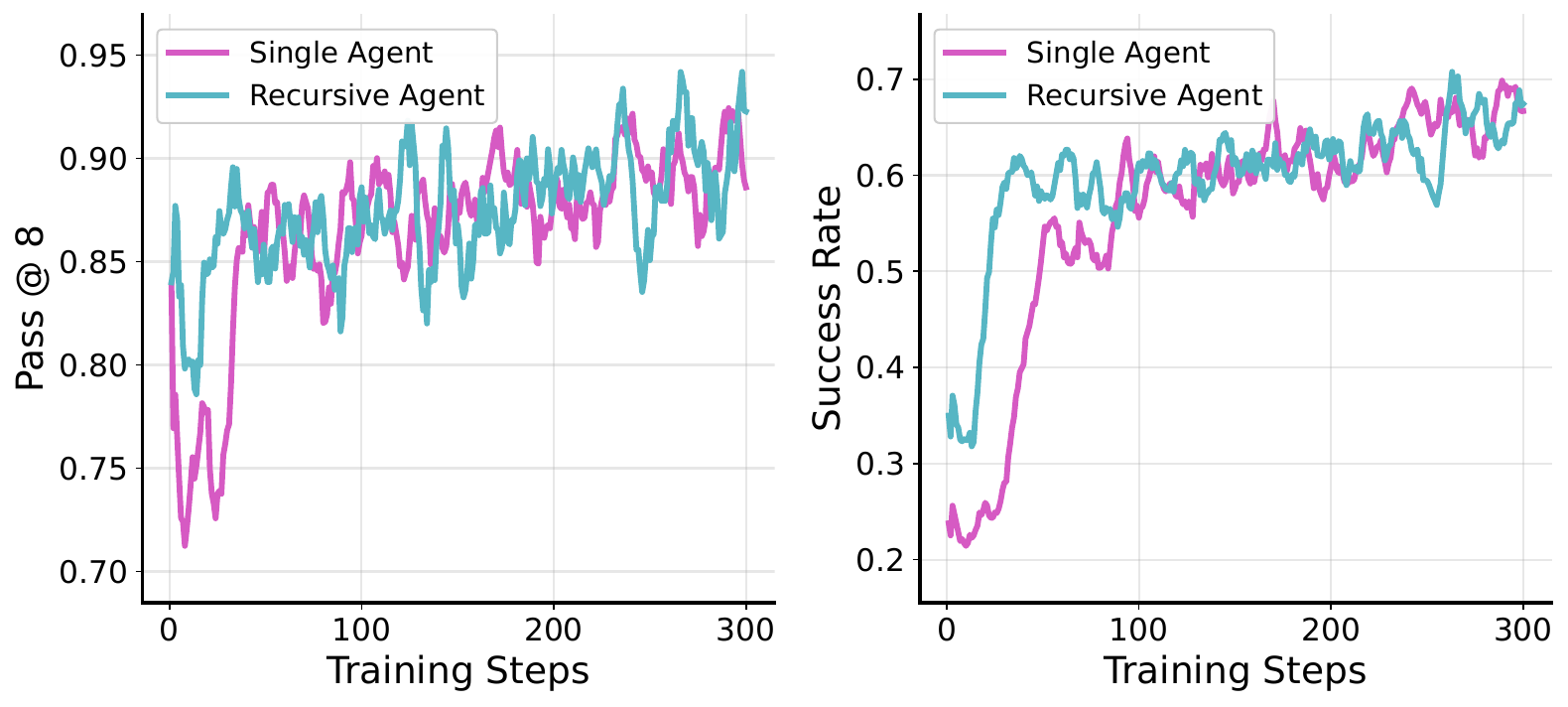}
   \caption{\arte{} (Email Search) training curves for Qwen-3-14B (moving average; window size 10).}
    \label{fig:email-curves}
\end{figure}

\subsection{Unbiased Baseline for \method}\label{app:lemma-unbiased}
\begin{lemma}[Unbiased leave-one-out baseline]
\label{lem:rao_unbiased}
Let $\tau^{(g)}$ be any trajectory in rollout tree $\mathcal{T}^{(g)}$ — root or sub-agent — and let $b_{-g}$ be the leave-one-out baseline defined in Eq.~\eqref{eq:rao_advantage}. Then
\[
\mathbb{E}\!\left[
\bigl(R(\tau^{(g)}) - b_{-g}\bigr)\,
\nabla_\theta \log \pi_\theta(\tau^{(g)})
\right]
=
\mathbb{E}\!\left[
R(\tau^{(g)})\,
\nabla_\theta \log \pi_\theta(\tau^{(g)})
\right],
\]
i.e., subtracting $b_{-g}$ does not bias the policy gradient estimator, even when $\tau^{(g)}$ is a sub-agent trajectory.
\end{lemma}
\begin{proof}
It suffices to show that the bias term introduced by the baseline vanishes:
\[
\mathbb{E}\!\left[
b_{-g}\,\nabla_\theta \log \pi_\theta(\tau^{(g)})
\right]
= 0.
\]
By construction, $b_{-g}$ is a function of the root rewards $\{R_{\mathrm{root}}^{(g')}\}_{g' \neq g}$ alone. Since the $G$ rollout trees are sampled independently, $b_{-g}$ is independent of $\tau^{(g)}$ and of the task assigned to the node that generated it, including for any sub-agent within rollout $g$. Therefore,
\[
\mathbb{E}\!\left[
b_{-g}\,\nabla_\theta \log \pi_\theta(\tau^{(g)})
\right]
=
\mathbb{E}[b_{-g}]\cdot
\mathbb{E}\!\left[
\nabla_\theta \log \pi_\theta(\tau^{(g)})
\right]
= 0,
\]
where the second factor vanishes by the score-function identity $\mathbb{E}\!\left[\nabla_\theta \log \pi_\theta(\tau^{(g)})\right] = 0$.
\end{proof}

\subsection{Environment Action Spaces}\label{app:action-space}
For each domain, the agent has access to standard python functionality via a python REPL implemented using \href{https://jupyter.org/}{Jupyter's} IPython kernel. On top of this, each domain has additional predefined functions available. We describe these in the tables below.

\begin{table}[h!]
\centering
\small
\vspace{-2mm}
\resizebox{\textwidth}{!}{
\begin{tabular}{l|l}
\toprule
\multicolumn{1}{c|}{\bf Action Type} & \multicolumn{1}{c}{\bf Description} \\
\midrule
{\tt craft(ingredients: dict, target: tuple[str, int]) -> str} & {Try crafting the target given the ingredients.} \\
{\tt get\_info(items: list) -> list[dict]} & {Get information about item(s), including crafting recipes.} \\
{\tt finish(message: str) -> str} & {Finish the rollout with a message string.} \\
{\tt launch\_subagent(targets: dict) -> str} & {Launch a subagent with some goal targets.}  \\
\bottomrule
\end{tabular}
}
\caption{\textcraft~environment actions.}
\label{tab:textcraft-action-space}
\end{table}

\begin{table}[h!]
\centering
\small
\vspace{-2mm}
\resizebox{\textwidth}{!}{
\begin{tabular}{l|l}
\toprule
\multicolumn{1}{c|}{\bf Action Type} & \multicolumn{1}{c}{\bf Description} \\
\midrule
{\tt launch\_subagent(goal: str, context: str) -> Any} & {Launch a subagent with a delegated goal.}  \\
\bottomrule
\end{tabular}
}
\caption{\oolong~environment actions.}
\label{tab:oolong-action-space}
\end{table}

\begin{table}[h!]
\centering
\small
\vspace{-2mm}
\resizebox{\textwidth}{!}{
\begin{tabular}{l|l}
\toprule
\multicolumn{1}{c|}{\bf Action Type} & \multicolumn{1}{c}{\bf Description} \\
\midrule
{\tt search\_web(query: str, max\_results: int = 5) -> dict} & {Search the web using the query.} \\
{\tt view\_webpage\_content(url: str) -> str} & {Get the content of a webpage.} \\
{\tt finish(message: str) -> str} & {Finish the rollout with a message string.} \\
{\tt launch\_subagent(goal: str) -> Any} & {Launch a subagent with a delegated goal.}  \\
\bottomrule
\end{tabular}
}
\caption{\deepdive~environment actions. All functions are asynchronous, allowing concurrent execution.}
\label{tab:deepdive-action-space}
\end{table}

\begin{table}[h!]
\centering
\small
\vspace{-2mm}
\resizebox{\textwidth}{!}{
\begin{tabular}{l|l}
\toprule
\multicolumn{1}{c|}{\bf Action Type} & \multicolumn{1}{c}{\bf Description} \\
\midrule
{\tt search\_emails(query: str, max\_results: int = 10) -> dict} & {Search a user's inbox using the query.} \\
{\tt read\_email(message\_id: str) -> str} & {Get the content of an email.} \\
{\tt finish(message: str) -> str} & {Finish the rollout with a message string.} \\
{\tt launch\_subagent(goal: str) -> Any} & {Launch a subagent with a delegated goal.}  \\
\bottomrule
\end{tabular}
}
\caption{\arte~environment actions. All functions are asynchronous, allowing concurrent execution.}
\label{tab:arte-action-space}
\end{table}

\subsection{Additional Experimentation Details}\label{app:exp-details}

\textcraft{} and \deepdive{} experiments were performed using the AReaL training backend \citep{areal} with a learning rate of 3e-6. For the single-agent baseline, we allow 200 and 100 steps per rollout during training for the two benchmarks, respectively, and 20K steps during evaluation. For recursive agents, we allow 25 steps for both the root and sub-agents. For \oolong~experiments we use Tinker~\citep{tinker} for training: LoRa with rank 32 and a learning rate of 3e-5. The single-agent baseline is allowed 50 steps per rollout, while in the recursive agent case, both root and sub-agents were allowed 15 steps. For all experiments we use a batch size of 16 tasks and a group size of 8 rollouts and use a CISPO-style objective \citep{cispo}. All training was done using asynchronous RL, allowing staleness of up to 3 batches. For evaluation experiments we use an 8xH200 node for inference: for \textcraft{} and \deepdive{}, we use data parallel across all 8 GPUs with a single rollout performing inference at a time, while for \oolong~we use tensor parallel across all 8 GPUs with 5 rollouts performing inference at a time. 

\subsection{Token Usage Comparison}
In the tables below, we compare token usage when benchmarking the recursive and single agents across the different domains. Cache read estimates are idealized, based on an assumption of perfect cache reads. On domains like \oolong{} and \deepdive{} where the single agent learns only simple and short-horizon strategies for solving the tasks, token consumption is much higher for the recursive agent which learns more effortful and long-horizon strategies that scale to harder problems. In contrast, on \textcraft{} where both the single agent and the recursive agent learn effective long-horizon strategies for solving tasks, the recursive agent actually ends up using much fewer tokens compared to the single agent (especially as we increase difficulty). This is because as task horizons increase, single agent token usage scales quadratically with the number of steps in the trajectory, since each step needs to re-encode previous history tokens (although much of this can be alleviated with KV-Caching); on the other hand, the recursive agent splits the trajectory into many shorter sub-trajectories which do not need to re-encode a long history for each new agent step performed.

\begin{table}[h]
\centering
\caption{\textcraft~ token usage on the intersection of tasks solved by both single and recursive agents in the 40K/256K setting.}
\label{tab:textcraft-40k-token-breakdown}
\footnotesize
\setlength{\tabcolsep}{3.5pt}

\begin{NiceTabular}{llrrrr}[cell-space-top-limit=1.5pt,cell-space-bottom-limit=1.5pt]
\toprule
Diff. & Method & Input & Output & Cache-Read & Total \\
\midrule
\Block[fill=evalAll!35]{2-1}{\centering All}
  & Single    & 1022K & 6K  & 1021K & 1029K \\
  & Recursive & 184K  & 12K & 153K  & 196K \\
\cmidrule(lr){1-6}

\Block[fill=evalEasy!35]{2-1}{\centering Easy}
  & Single    & 24K & 2K & 23K & 25K \\
  & Recursive & 32K & 2K & 26K & 34K \\
\cmidrule(lr){1-6}

\Block[fill=evalMedium!35]{2-1}{\centering Medium}
  & Single    & 768K & 8K  & 767K & 775K \\
  & Recursive & 169K & 11K & 141K & 181K \\
\cmidrule(lr){1-6}

\Block[fill=evalHard!35]{2-1}{\centering Hard}
  & Single    & 8594K & 26K & 8593K & 8619K \\
  & Recursive & 1153K & 72K & 965K  & 1225K \\
\bottomrule
\end{NiceTabular}
\end{table}

\begin{table}[h]
\centering
\caption{\textsc{Oolong-real} token usage on the intersection of tasks with non-zero score for both single and recursive agents.}
\label{tab:oolong-real-token-breakdown}
\footnotesize
\setlength{\tabcolsep}{3.5pt}

\begin{NiceTabular}{llrrrr}[cell-space-top-limit=1.5pt,cell-space-bottom-limit=1.5pt]
\toprule
Bucket & Method & Input & Output & Cache-Read & Total \\
\midrule
\Block[fill=evalAll!35]{2-1}{\centering All}
  & Single    & 14.2K  & 0.9K  & 13.7K  & 15.1K   \\
  & Recursive & 257.3K & 67.5K & 243.5K & 324.8K  \\
\cmidrule(lr){1-6}

\Block[fill=evalEasy!35]{2-1}{\centering 55K}
  & Single    & 11.0K  & 0.7K  & 10.6K  & 11.8K   \\
  & Recursive & 119.3K & 27.3K & 112.4K & 146.6K  \\
\cmidrule(lr){1-6}

\Block[fill=evalMedium!35]{2-1}{\centering 118K}
  & Single    & 9.9K   & 0.9K  & 9.4K   & 10.8K   \\
  & Recursive & 244.1K & 64.1K & 230.8K & 308.2K  \\
\cmidrule(lr){1-6}

\Block[fill=evalHard!35]{2-1}{\centering 175K}
  & Single    & 20.3K  & 1.0K  & 19.9K  & 21.3K   \\
  & Recursive & 365.1K & 98.5K & 346.0K & 463.6K  \\
\bottomrule
\end{NiceTabular}
\end{table}

\begin{table}[h]
\centering
\caption{\textsc{DeepDive} token usage on the intersection evaluation tasks solved by both methods.}
\label{tab:deepdive-token-breakdown}
\footnotesize
\setlength{\tabcolsep}{3.5pt}

\begin{NiceTabular}{lrrrr}[cell-space-top-limit=1.5pt,cell-space-bottom-limit=1.5pt]
\toprule
Method & Input & Output & Cache-Read & Total \\
\midrule

Single    & 19.1K  & 0.8K  & 18.2K  & 19.9K  \\
Recursive & 538.4K & 23.0K & 499.5K & 561.4K \\
\bottomrule
\end{NiceTabular}
\end{table}

\subsection{TextCraft-Synth Prompts}

\newtcolorbox{promptbox}[1][]{
  enhanced,
  breakable,
  colback=promptbg,
  colframe=promptred,
  coltitle=white,
  fonttitle=\bfseries,
  title=#1,
  boxrule=1pt,
  arc=3mm,
  left=6mm,
  right=6mm,
  top=4mm,
  bottom=4mm
}

\begin{promptbox}[TextCraft-Synth Single Agent System Prompt]
You are an agent in a crafting game.
Your goal is to craft items by combining ingredients.
You have access to an inventory of existing ingredients, which are sufficient to craft the target items; though, you may need to craft intermediate ingredients first.\\

Note: If you already have one of the target items in your inventory, you should craft the requested number of the target on top of what you already have.\\
For example, if you already have 2 wooden\_pickaxes but your goal is to craft 3, your inventory should end up with 5 wooden\_pickaxes.\\

CRAFTING STRATEGY:\\
- Recipes produce fixed quantities per execution - you cannot craft arbitrary amounts \\
  Example: If a recipe produces 2 items, you can only craft in multiples of 2 (2, 4, 6...)\\
- Recipe ingredients scale with the number of times you execute it
  Example: Recipe "2 ore → 2 items" means 2 ore for 1 execution, 4 ore for 2 executions\\
- Always verify what you have before claiming something is impossible\\
- Check your inventory and recipe information to confirm ingredient availability\\
- Calculate carefully: if a recipe uses 2 ingredients to make 2 items, you need exactly 2 ingredients for 2 items\\
\\
You can perform actions by writing Python code blocks. You will get multiple steps to complete the task.
For your current step, first briefly reason (~1-3 sentences) about your strategy in\texttt{<thought> </thought>} tags, then output your code in \texttt{<code> </code>} tags.

Your code will be executed in a Jupyter notebook and the output will be shown to you.
\end{promptbox}

\begin{promptbox}[TextCraft-Synth Recursive Agent System Prompt]
You are an agent in a crafting game.\\
Your goal is to craft items by combining ingredients.\\
You have access to an inventory of existing ingredients, which are sufficient to craft the target items; though, you may need to craft intermediate ingredients first.\\

Note: If you already have one of the target items in your inventory, you should craft the requested number of the target on top of what you already have.\\
For example, if you already have 2 wooden\_pickaxes but your goal is to craft 3, your inventory should end up with 5 wooden\_pickaxes.\\

\texttt{\textless TIPS\textgreater}\\
CRAFTING STRATEGY:\\

- Recipes produce fixed quantities per execution - you cannot craft arbitrary amounts\\
\hspace*{1em}Example: If a recipe produces 2 items, you can only craft in multiples of 2 (2, 4, 6...)\\
- Recipe ingredients scale with the number of times you execute it\\
\hspace*{1em}Example: Recipe ``2 ore $\rightarrow$ 2 items'' means 2 ore for 1 execution, 4 ore for 2 executions\\
- Always verify what you have before claiming something is impossible\\
- Check your inventory and recipe information to confirm ingredient availability\\
- Calculate carefully: if a recipe uses 2 ingredients to make 2 items, you need exactly 2 ingredients for 2 items\\

DELEGATION STRATEGY:\\

- It is \textbf{highly recommended} to delegate crafting of intermediate ingredients\\
- Break complex tasks into INDEPENDENT subtasks that can be solved separately\\
- For tasks that are sufficiently complex, it is recommended to recursively delegate; i.e., subagents can further delegate to other subagents.\\
- Delegate one group of related items at a time, not everything at once\\
- Items can be delegated in parallel if they don't depend on each other\\
- Delegated tasks share your inventory - results are immediately available\\
- \textbf{IMPORTANT}: \texttt{launch\_subagent} is an async function, you MUST use \texttt{await}:\\
\hspace*{1em}* CORRECT: \texttt{await launch\_subagent(\{"item": 1\})}\\
\hspace*{1em}* CORRECT: \texttt{await asyncio.gather(launch\_subagent(...), launch\_subagent(...))}\\
\hspace*{1em}* WRONG: \texttt{launch\_subagent(\{"item": 1\})} -- missing \texttt{await}, will error\\

\texttt{\textless/TIPS\textgreater}\\

You can perform an action by writing a block of code. You will get multiple steps to complete the task.\\
For your current step, first briefly reason (\textasciitilde1--3 sentences) about your next step in the \texttt{\textless thought\textgreater\ \textless/thought\textgreater} tags and then output your code action in \texttt{\textless python\textgreater\ \textless/python\textgreater} tags.\\
Your code cell will be executed inside a Jupyter notebook and the output will be shown to you.
\end{promptbox} 

\begin{promptbox}[TextCraft-Synth User Prompt (Same for both)]
Craft the following items: \texttt{<count>x <item>[, <count>x <item>, ...]}\\
Example:\\
Craft the following items: \texttt{1x m5\_i3, 2x c3\_i2}
\end{promptbox}

\begin{promptbox}[TextCraft-Synth Single Agent Action Space]
\begin{verbatim}
1. def craft(ingredients: dict, target: tuple[str, int]) -> str
   Craft items using ingredients from your inventory.
   - ingredients: Dict of item_name: count to consume
   - target: (item_name, total_count) where total_count must be 
        divisible by recipe result_count
   - Example: craft({"m0_i1": 2, "m1_i1": 1}, ("m2_i2", 2))

2. def get_info(items: list) -> list[dict]
   Get recipe information for items.
   - Returns: List with {"item": str, "can_craft": bool, 
        "is_base": bool, "in_inventory": int, "crafting_depth": int, 
        "recipes": [...]}
   - crafting_depth indicates complexity: 0=base item, 
        1=direct craft, 2+=needs intermediate steps
   - Each recipe shows {"ingredients": {...}, "result_count": int}
   - Example: get_info(["m2_i2", "raw_m0"])

3. def view_inventory() -> dict
   View your current inventory.
   - Returns: Dict of {item_name: count}
   - Example: inv = view_inventory()

4. def finish(message: str) -> str
   Complete the task.
   - Example: finish("Successfully crafted all required items")
\end{verbatim}
\end{promptbox}

\begin{promptbox}[TextCraft-Synth Recursive Agent Action Space]
\begin{verbatim}
1. def craft(ingredients: dict, target: tuple[str, int]) -> str
   Craft items using ingredients from your inventory.
   - ingredients: Dict of {item_name: count} to consume
   - target: (item_name, total_count) where total_count must be 
        divisible by recipe result_count
   - Example: craft({"m0_i1": 2, "m1_i1": 1}, ("m2_i2", 2))

2. def get_info(items: list) -> list[dict]
   Get recipe information for items.
   - Returns: List with {"item": str, "can_craft": bool, "is_base": bool, 
        "in_inventory": int, "crafting_depth": int, "recipes": [...]}
   - crafting_depth indicates complexity: 0=base item, 1=direct craft, 
        2+=needs intermediate steps
   - Each recipe shows {"ingredients": {...}, "result_count": int}
   - Example: get_info(["m2_i2", "raw_m0"])

3. def view_inventory() -> dict
   View your current inventory.
   - Returns: Dict of {item_name: count}
   - Example: inv = view_inventory()

4. def finish(message: str) -> str
   Complete the task.
   - Example: finish("Successfully crafted all required items")

5. async def launch_subagent(targets: dict, num_steps: int, 
    context: str = "") -> str
   Launch a subagent to craft specific targets (shares your inventory).
   - targets: Dict of {item_name: count} to craft
   - num_steps: Budget for subagent
     * Use crafting_depth from get_info() to estimate: depth × 8-10 steps
     * Example: depth=4 needs ~32-40 steps, depth=8 needs ~64-80 steps
   - context: Optional context string for the subagent
   - Sequential: await launch_subagent({"m0_i2": 4}, 20)
   - Parallel: results = await asyncio.gather(
       launch_subagent({"m0_i2": 4}, 20),
       launch_subagent({"c1_i2": 3}, 15)
     )
   - Returns: Subagent's finish message (or list if using gather)

Note: asyncio is already imported.
Use await asyncio.gather(...) to run subtasks in parallel or 
await launch_subagent() for a single subtask. Do not forget 
to await the results.
\end{verbatim}
\end{promptbox}

\subsection{Oolong-Real Prompts}\label{app:oolong-prompts}
\newtcolorbox{oolongpromptbox}[1][]{
  enhanced,
  breakable,
  colback=evalheader,
  colframe=promptblue,
  coltitle=white,
  fonttitle=\bfseries,
  title=#1,
  boxrule=1pt,
  arc=3mm,
  left=6mm,
  right=6mm,
  top=4mm,
  bottom=4mm
}
\begin{oolongpromptbox}[Oolong-Real Single Agent System Prompt]
You are tasked with answering a query that requires analyzing and aggregating information from a large context.\\

You have access to a REPL environment with the following pre-loaded variable:\\
- `context` (str): The full text context to analyze (may be very large)\\
\\
\texttt{<TIPS>}

CONTEXT ANALYSIS:

- If the length on the context is very large (\>32K characters), first examine/peek into the structure of the context (what format is the data in?)

- The context may be structured in a way that the task can be solved via programmatic parsing or matching. **Only take this approach if you are 100\% sure that you can reliably solve the task this way.**

- For the majority of cases, you should print out the context with `print(context)` to observe and read it manually and answer the question by reading the context.

ANSWER SUBMISSION:

- You can submit your answer using the `finish` function in the format requested in the user provided goal.

\texttt{</TIPS>}\\
\\
You can perform actions by writing Python code blocks. You will get multiple steps to complete the task.
For your current step, first briefly reason (~1-3 sentences) about your strategy in \texttt{<thought> </thought>} tags, then output your code in \texttt{<python> </python>} tags.\\
\\
Your code will be executed in a Jupyter notebook and the output will be shown to you.
\end{oolongpromptbox}

\begin{oolongpromptbox}[Oolong-Real Recursive Agent System Prompt]
You are tasked with answering a query that requires analyzing and aggregating information from a large context.\\
You have access to a REPL environment with the following pre-loaded variable:\\
- `context` (str): The full text context to analyze (may be very large)\\
\\
\texttt{<TIPS>}\\
CONTEXT ANALYSIS:\\
- First check if the length on the context is very large (\texttt{>}32K characters) using `len(context)`.\\
- For very large contexts (i.e., \texttt{>}32K characters), work with chunks rather than the entire context at once.\\
- Use subagents to process chunks and then aggregate the results to produce a final answer. Try not to split the context into too many chunks (32K characters per chunk is a good rule of thumb)\\
- If the context \texttt{<=} 32K characters, prefer to process your context by printing out and reading it rather than using programmatic heuristics.\\
- **IMPORTANT: DO NOT USE regex, string matching, etc. types of programmatic heuristics. It is important to read the context with `print(context)` to be accurate in your answer.**\\
SUBAGENT DELEGATION:\\
-**IMPORTANT: Do not use subagents if the context you need to process is \texttt{<=} 32K characters. Just print out the context to observe it directly and answer the question by reading the context.**\\
- You have the ability to spawn subagents (other instantiations of yourself), by providing them with their own `context`/chunk to process and a goal/instruction for what result it should return.\\
- You can use `asyncio.gather` to process multiple chunks simultaneously.\\
- Be specific about the format and type in which you expect subagents to return their results.\\
- Do not provide the context/chunk as part of the goal. Instead, pass it explicitly as the `context` argument to the `launch\_subagent` function.\\
ANSWER SUBMISSION:
- You can submit your answer using the `finish` function in the format requested in the user provided goal.\\
\texttt{</TIPS>}\\

You can perform printing out or peaking into the context or launching sub-agents using Python code blocks. You will get multiple steps to complete the task.
For your current step, first briefly reason (~1-3 sentences) about your recursive strategy in \texttt{<thought> </thought>} tags, then output your code in \texttt{<python> </python>} tags.\\
\\
Your code will be executed in a Jupyter notebook and the output will be shown to you. The python code block should be formatted as follows: \texttt{<python>code block</python>} without any other tags.\\
\\
Do not output anything else except for \texttt{<thought>...</thought>}\\
\texttt{<python>...</python>}
\end{oolongpromptbox}

\begin{oolongpromptbox}[Oolong-Real User Prompt (same for both)]
\texttt{<question from the dataset>}

Notes:\\
- The task/user prompt is exactly the dataset question.\\
- The long document to analyze is provided separately as the pre-loaded 
    REPL variable `context`.
\end{oolongpromptbox}

\begin{oolongpromptbox}[Oolong-Real Single Agent Action Space]
\begin{verbatim}
Available Actions (python functions):

Pre-loaded variable:
- context (str): The full text context to analyze

1. def finish(message: str) -> str
   Complete the task with your answer.
\end{verbatim}
\end{oolongpromptbox}

\begin{oolongpromptbox}[Oolong-Real Recursive Agent Action Space]
\begin{verbatim}
Available Actions (python functions):

Pre-loaded variable:
- context (str): The full text context to analyze

1. async def launch_subagent(goal: str, context: str) -> Any
   Launch a subagent to process a chunk of the context.
   Returns the result of the subagent's execution with return type 
   specified in the goal
   (if specified, else str).
   - goal: The goal/instruction for the subagent. Tell the subagent what 
    information you want and specify the exact format and type in which 
    you expect the subagent to return its result.
   - context: The chunk of the context to process.

   Example:
   result = await launch_subagent(
       goal="Find the last two spells mentioned in the context by printing 
       out the context and manually reading it. Return the spells as a 
       stringified list of spells that I can parse with json.loads()",
       context=context_chunk
   )
   print(f"Subagent result: {result}")

   Note: `asyncio` is already imported. Use `await asyncio.gather(...)` to 
   run subtasks in parallel or `await launch_subagent()` for a single subtask. 
   Do not forget to await the function call.

2. def finish(result: Any) -> Any
   Complete the task with your result.
\end{verbatim}
\end{oolongpromptbox}

\begin{oolongpromptbox}[Oolong-Real LLM Judge Prompt for sub-agent rewards]
We need to judge the performance of an agent on a task. \\The task relates to aggregating some information from a potentially very large context.\\
1. Read the context carefully and see if the agent's answer is accurate. \\
2. Do not mark the agent as successful unless it prints out the context and \\reads it manually or alternatively uses subagents to answer the question.\\
For example, if the agent uses regex or string matching/contains logic to\\ answer the question, this is a heuristic that may not be reliable in general\\ and thus should not be marked as successful.\\ \\
Using subagents is okay since subagents can view the context on behalf of the agent.\\
Please provide a reason and success flag (boolean value) in the following format:\\
\begin{verbatim}
```
{
  "reason": "Brief reasoning for success flag here.",
  "success": "True" or "False",
}
```
\end{verbatim}
\end{oolongpromptbox}

\subsection{DeepDive Prompts}\label{app:deepdive-prompts}

\newtcolorbox{deepdivepromptbox}[1][]{
  enhanced,
  breakable,
  colback=evalEasy,
  colframe=promptgreen,
  coltitle=white,
  fonttitle=\bfseries,
  title=#1,
  boxrule=1pt,
  arc=3mm,
  left=6mm,
  right=6mm,
  top=4mm,
  bottom=4mm
}

\begin{deepdivepromptbox}[DeepDive Single Agent System Prompt]
You are a deep research agent solving a factual question by searching the web.\\

You have access to Python plus web-search tools. Use them deliberately:\\
- Start broad, then refine based on what you learn.\\
- Cross-check key claims across multiple sources when possible.\\
- Use \texttt{await view\_webpage\_content(url)} when snippets are insufficient or you need detailed evidence.\\
- Avoid dumping huge webpage bodies into the notebook unless necessary.\\
- Keep intermediate notes concise and use Python to organize findings.\\

ANSWER SUBMISSION:\\
- When you are confident, call \texttt{finish(...)}.\\
- The final answer should directly answer the question and stay concise unless the task explicitly asks for more detail.\\

OTHER TIPS:\\
- \textbf{All functions except for finish are async functions. YOU MUST AWAIT THE RESULTS OF THESE FUNCTIONS}\\

You can perform actions by writing Python code blocks. You will get multiple steps to complete the task.
For your current step, first briefly reason (~1-3 sentences) in \texttt{<thought> </thought>} tags, then output code in \texttt{<python> </python>} tags.\\
Your code will be executed in a Jupyter notebook and the output will be shown to you.
\end{deepdivepromptbox}

\begin{deepdivepromptbox}[DeepDive Recursive Agent System Prompt]
You are a deep research agent solving a factual question by searching the web.\\

You have access to Python plus web-search tools, and you can delegate subproblems to subagents.\\

RESEARCH STRATEGY:\\
- Break the question into a small number of meaningful subquestions.\\
- Search broadly first, then narrow onto the most promising sources.\\
- Cross-check important claims across multiple sources when possible.\\
- Use \texttt{await view\_webpage\_content(url)} when search snippets are not enough.\\
- Use Python to store notes, compare evidence, and synthesize findings.\\

DELEGATION STRATEGY:\\
- You have the ability to spawn subagents and delegate subtasks to them. Make effective use of subagents to solve the task!\\
- Use \texttt{await launch\_subagent(goal)} for coherent subproblems such as source discovery, fact verification, or answering one component of a multi-hop question.\\
- Tell subagents exactly what to return, including format when useful.\\
- Subagents can run in parallel with \texttt{await asyncio.gather(...)}.\\
- Subagents can themselves delegate recursively.\\

ANSWER SUBMISSION:\\
- When you are confident, call \texttt{finish(...)}.\\
- The final answer should directly answer the question and stay concise unless the task explicitly asks for more detail.\\

OTHER TIPS:\\
- \textbf{All functions except for finish are async functions. YOU MUST AWAIT THE RESULTS OF THESE FUNCTIONS}\\

You can perform actions by writing Python code blocks. You will get multiple steps to complete the task.
For your current step, first briefly reason (~1-3 sentences) about your research or delegation strategy in \texttt{<thought> </thought>} tags, then output code in \texttt{<python> </python>} tags.\\
Your code will be executed in a Jupyter notebook and the output will be shown to you.
\end{deepdivepromptbox}

\begin{deepdivepromptbox}[DeepDive User Prompt (same for both)]
\texttt{<question from the dataset>}\\

Notes:\\
- The task/user prompt is exactly the dataset question.\\
- The ground-truth answer is used only for evaluation and is not shown to the agent.
\end{deepdivepromptbox}

\begin{deepdivepromptbox}[DeepDive Single Agent Action Space]
\begin{verbatim}
Available Actions (python functions):

1. async def search_web(query: str, max_results: int = 5) -> dict
   Search the web for information related to the query.
   - query: The query to search for.
   - max_results: Optional maximum number of results to return.
     Must be between 1 and 20. Defaults to 5.

   Returns a dictionary containing:
   {
       "query": str,
       "follow_up_questions": list[str],
       "answer": str,
       "images": list[str],
       "results": list[dict],
       "response_time": float,
       "request_id": str,
   }

   Each result contains:
   {
       "url": str,
       "title": str,
       "content": str,
       "score": float,
       "raw_content": str | None,
   }

2. async def view_webpage_content(url: str) -> str
   View the content of a webpage.
   - url: The URL of the webpage to view.

   Returns a string containing the webpage content. This may be very long,
   so it is often useful to inspect its size before printing the full text.

3. def finish(message: str) -> str
   Complete the task with your answer.
   This is synchronous and should not be awaited.
\end{verbatim}
\end{deepdivepromptbox}

\begin{deepdivepromptbox}[DeepDive Recursive Agent Action Space]
\begin{verbatim}
Available Actions (python functions):

1. async def launch_subagent(goal: str) -> Any
   Launch a subagent to solve a subtask.
   - goal: The instruction for the subagent. This can be a simple or
     compound task. Subagents can recursively delegate tasks to other
     subagents.
   - Specify the format and type of answer you expect from the subagent.

   Example:
   ps5_price_range = launch_subagent(
       "Find the price range of a PS5 across sony, bestbuy, amazon and
       gamestop. Return the answer as a string of the form '$$$ - $$$$'."
   )
   switch2_price_range = launch_subagent(
       "Find the price range of a Switch 2 across nintendo, bestbuy,
       amazon and gamestop. Return the answer as a string of the form
       '$$$ - $$$$'."
   )
   results = await asyncio.gather(
       ps5_price_range,
       switch2_price_range,
       return_exceptions=True,
   )

   Note: asyncio is already imported. Use await asyncio.gather(...) to run
   subagents in parallel or await launch_subagent(goal) for a single
   subagent. Do not forget to await the results.

2. async def search_web(query: str, max_results: int = 5) -> dict
   Search the web for information related to the query.
   - query: The query to search for.
   - max_results: Optional maximum number of results to return.
     Must be between 1 and 20. Defaults to 5.

3. async def view_webpage_content(url: str) -> str
   View the content of a webpage.
   Returns a string containing the webpage content. This may be very long.

4. def finish(message: str) -> str
   Complete the task with your answer.
   This is synchronous and should not be awaited.
\end{verbatim}
\end{deepdivepromptbox}

\begin{deepdivepromptbox}[DeepDive LLM Judge Prompt (Root Task)]
We need to judge the performance of a deep research agent on a task. The task requires searching the web for information across various sources and synthesizing information together to answer a question.\\
The agent may use subagents to solve parts of the task. Do not penalize the model for relying on subagents, unless the subtasks delegated to the subagents are not meaningful or useful for the task.\\
You will be given the ground truth answer to the task and the agent's answer to the task.\\
When comparing the agent's answer to the ground truth answer, it is acceptable to have minor formatting differences as long as the core information is equivalent.\\
Please provide a reason and success flag (boolean value) in the following format:\\
\begin{verbatim}
```json
{
    "reason": "Brief reasoning for success flag here.",
    "success": true or false
}
```
\end{verbatim}
\end{deepdivepromptbox}

\begin{deepdivepromptbox}[DeepDive LLM Judge Prompt (Sub-Task)]
We need to judge the performance of a deep research agent on a sub-task. The sub-task is part of a larger task that requires searching the web for information across various sources and synthesizing information together to answer a question.\\
The agent may use subagents to solve parts of the sub-task. Do not penalize the model for relying on subagents when the delegated subtasks are meaningful and useful.\\
You should mark the sub-task as successful if the agent's final answer is correct and, if the agent delegates, its delegation strategy is useful and efficient.\\
Reward delegation only when it is genuinely useful: the subtasks should be concrete, non-overlapping, and should help the agent search, verify, or synthesize information more effectively.\\
Mark the sub-task as unsuccessful for degenerate delegation behavior even if the final answer is correct. Degenerate behavior includes repeatedly forwarding nearly the same or whole goal to subagents without meaningful decomposition, search, or reading work, or wasteful failed launches caused by trying to delegate past the depth limit.\\
It is acceptable for the agent to do all the work itself, part of the work itself, or to use subagents heavily if those subagents do distinct useful work.\\
Do not give credit just because the wording of child tasks changes slightly from the agent's own goal; judge whether the decomposition is actually useful and efficient.\\
Please provide a reason and success flag (boolean value) in the following format:\\
\begin{verbatim}
```json
{
    "reason": "Brief reasoning for success flag here.",
    "success": true or false
}
```
\end{verbatim}
\end{deepdivepromptbox}

\end{document}